\documentclass[journal]{IEEEtran}
\usepackage[bibstyle=ieee,citestyle=numeric-comp,sortcites=true,maxnames=4,minnames=3,maxbibnames=3]{biblatex}

\usepackage{amsfonts}
\usepackage[cmex10]{amsmath}
\usepackage{amsthm}
\usepackage{pgfplots}
\usepackage{tikz}
\usetikzlibrary{arrows.meta, positioning, fit}

\makeatletter
\g@addto@macro\normalsize{%
  \setlength\abovedisplayskip{2pt}%
  \setlength\belowdisplayskip{2pt}%
  \setlength\abovedisplayshortskip{2pt}%
  \setlength\belowdisplayshortskip{2pt}%
}
\makeatother

\theoremstyle{definition}
\newtheorem{definition}{Definition}
\newtheorem{assumption}{Assumption}
\newtheorem{remark}{Remark}
\newtheorem{theorem}{Theorem}

\theoremstyle{remark}

\usepackage{braket}
\usepackage{booktabs}

\usepackage{array}
\usepackage{makecell}
\usepackage{mdwmath}
\usepackage{mathrsfs}
\usepackage{mdwtab}
\usepackage{eqparbox}
\usepackage[caption=1]{caption}
\usepackage{stfloats}
\usepackage{url}
\usepackage{amssymb}
\usepackage{float}
\usepackage{indentfirst}
\usepackage{makeidx}
\usepackage{tabularx}
\usepackage{cases}
\usepackage{bbm}
\usepackage{algorithmic}
\usepackage{mathtools}
\usepackage{xcolor}
\usepackage[normalem]{ulem}  
\newcommand{\new}[1]{#1}            
\newcommand{\del}[1]{}  
\usepackage{bm}
\usepackage{stfloats}
\usepackage{lipsum}
\usepackage{amsfonts}
\usepackage{bm}
\usepackage{dsfont}
\usepackage{lipsum}
\usepackage{graphicx}
\usepackage{subcaption}
\usepackage{url}
\usepackage{bbm}
\newcommand{\RNum}[1]{\uppercase\expandafter{\romannumeral #1\relax}}
\usepackage{multirow}
\allowdisplaybreaks
\makeatletter

\usepackage[linesnumbered,ruled,vlined]{algorithm2e}
\usepackage{amsmath}
\makeatother

\usepackage{amsthm}
\theoremstyle{definition} 
\theoremstyle{remark} 
\pgfplotsset{compat=1.18}
\begin{document}
\title{FedFit: Federated Fine-Tuning of LLMs via Vector-Bank Parameterization and Quantization}

\author{
Hang~Zou,
Chao~Zhang,
Yuzhi~Yang,
Yu~Tian,
Samson~Lasaulce,
and~M\'erouane~Debbah

\thanks{H. Zou, Y. Yang, Y. Tian and M. Debbah are with 6G Research Center, Khalifa University, 127788 Abu Dhabi, UAE (e-mails: \{hang.zou, yuzhi.yang, yu.tian, merouane.debbah\}@ku.ac.ae).
}
\thanks{C. Zhang (corresponding author) is with Department of Computer Science, Central South University, 410083, Changsha, China (email: chao.zhang@csu.edu.cn)}
\thanks{S. Lasaulce is with Universit\'{e} de Lorraine, CNRS, CRAN, F-54000 Nancy, France (email: samson.lasaulce@univ-lorraine.fr).}}

\maketitle
\begin{abstract}
Federated Learning (FL) enables privacy-preserving fine-tuning of Large Language Models (LLMs), yet the massive communication overhead remains a critical bottleneck. Furthermore, applying Low-Rank Adaptation (LoRA) in FL faces a fundamental ``aggregation dilemma" between the accurate Sum-of-Products (SoP) and the communication-efficient Product-of-Sums (PoS) implementations. To tackle these challenges, we propose \textbf{FedFit}. First, to significantly reduce communication overhead, we introduce a \textit{disjoint shared vector-bank parameterization} that reconstructs high-dimensional adapter matrices from two compact and disjoint global vector banks. Second, to address the aggregation dilemma, we devise an \textit{alternating optimization schedule}. By cycling between decoupled single-bank updates (which allow for accurate aggregation) and joint updates corrected by a Residual Spectral Aggregation mechanism, we resolve the conflict between SoP and PoS. Additionally, we integrate blockwise quantization with client-side error feedback to further compress the transmitted vectors. Furthermore, we establish theoretical convergence guarantees for the proposed algorithm. Extensive experiments on Qwen2.5 models demonstrate that FedFit achieves perplexity performance comparable to standard federated LoRA methods, while providing compression ratios up to $100\times$ higher.
\end{abstract}

\section{Introduction}

\subsection{Motivation}

Large language models (LLMs) have achieved remarkable success across a wide range of natural language processing tasks~\cite{brown2020language}, and fine-tuning LLMs for downstream tasks has become standard practice across many application domains. However, in privacy- or security-sensitive areas such as healthcare and finance, centralized fine-tuning is often infeasible because it requires uploading raw user data to a central server~\cite{kairouz2021advances}. In addition, in remote Internet-of-Things (IoT) or mobile scenarios, transmitting local data to the cloud and then broadcasting updated models back to devices can consume excessive wireless resources, which is particularly problematic in power- or bandwidth-limited environments~\cite{chen2021communication,xiao2024otafl,wang2024fedftwireless}. Federated learning (FL) addresses these issues by training models collaboratively across distributed clients using their local data, while only communicating local model updates to a central server for aggregation~\cite{kairouz2021advances, mcmahan2017communication}. Applying FL to LLM fine-tuning, often referred to as FedLLM, has attracted increasing attention~\cite{yao2024fedllm,su2024haflq,cheng2024towardfedllm,Wu2025Survey,zhang2023fedgpt,Xu2023FwdLLM}, yet the extremely large model size of LLMs makes full-parameter FedLLM prohibitively expensive in both computation and communication~\cite{qin2024fulltuning}.

Low-rank adaptation (LoRA), which consists in perturbing the weight matrix by a low-rank product, is a representative parameter-efficient fine-tuning (PEFT) method that can substantially reduce local computation, storage, and communication costs when fine-tuning LLMs~\cite{houlsby2019adapter,hu2022lora}. Instead of updating the full weight matrix of each layer, LoRA introduces a low-rank adapter of the form $\Delta W = BA^{\top}$, where $A$ and $B$ are two low-rank matrices~\cite{hu2022lora}. While this design works well in centralized fine-tuning, directly deploying LoRA in FedLLM raises a fundamental aggregation dilemma~\cite{bai2024flexlora,wang2024flora,singhal2024fedexlora,ghiasvand2025fedtt,liu2025ecolora}. In FedLLM, client $k$ learns its own adapters $(A_k,B_k)$ and uploads them to the server. The server must then combine these client-side updates into a single global adapter. If the server averages the products $\Delta W_k = B_k A_k^\top$ across clients (Sum-of-Product, SoP), the resulting update generally loses the exact low-rank structure and cannot be represented by a single pair of low-rank matrices, thereby increasing the downlink communication cost. Alternatively, if the server first averages $A_k$ and $B_k$ separately and then multiplies the two averaged matrices (Product-of-Sum, PoS), the LoRA format is preserved, but the aggregation becomes inaccurate because all cross-client cross terms are discarded. Recent works observed that these two aggregation rules can lead to different performance in LoRA-based FL~\cite{bai2024flexlora,wang2024flora,singhal2024fedexlora,liu2025ecolora,zhu2024quantlora,hu2024fedqlora}, but a systematic investigation or theoretical analysis of the SoP vs PoS discrepancy for LoRA-based FedLLM, especially under vector-bank parameterization and aggressive quantization are missing.

Although LoRA significantly reduces the computation and communication cost of FedLLM relative to full-parameter fine-tuning, LoRA adapters are still highly redundant across layers and clients~\cite{buyukakyuz2024olora,li2024vblora,li2025unilora}. Recent studies show that LoRA weights can be represented in a shared latent vector space with sparse indexing and low-rank structure, enabling further compression~\cite{li2024vblora,li2025unilora}. However, these vector-bank LoRA techniques have so far been studied only in centralized or single-client settings. To the best of our knowledge, there is currently no work that applies such vector-bank LoRA representations to a decentralized setting such as FedLLM.

Communication efficiency remains a critical bottleneck in realistic federated systems~\cite{chen2021communication,alistarh2017qsgd,reisizadeh2020fedpaq}. While LoRA already reduces the number of parameters to be transmitted, repeatedly communicating dense adapter matrices or large latent banks can still dominate the cost in wireless or edge environments~\cite{wang2024fedftwireless,ghiasvand2025fedtt,liu2025ecolora}. Quantization has also been widely adopted in FL to alleviate communication overhead~\cite{alistarh2017qsgd,reisizadeh2020fedpaq,karimireddy2019errorfeedback,li2023fedef}, and some recent methods have started to explore quantized LoRA in federated settings~\cite{zhu2024quantlora,hu2024fedqlora}. Nevertheless, there is still no work that investigates quantized vector-bank LoRA in FedLLM, where both the latent bank structure and the SoP/PoS aggregation trade-off must be handled jointly.
To make the bandwidth pressure concrete: even after LoRA has reduced the trainable parameter count by two orders of magnitude relative to full fine-tuning, a single per-round upload of a $\sim$15~MB adapter still takes roughly 24~seconds over a 5~Mbps cellular uplink, and proportionally saturates the link in cross-silo deployments with hundreds or thousands of participating clients. Communication, not the local compute imposed by quantized weights, therefore remains the binding constraint in the wireless and cross-organization FedLLM regimes targeted by this work, motivating further compression on top of LoRA.

These observations motivate the work in this paper: we design a FedLLM framework that (i) resolves the SoP–PoS aggregation dilemma, (ii) shares vector-bank parameterization across layers and clients, and (iii) supports aggressive quantization for ultra-low communication cost.

\subsection{Related Works}

\textbf{LLMs and Federated Learning:} Transformers and LLMs have become the dominant backbone for modern AI systems \cite{brown2020language}.
FL provides a standard paradigm for privacy-preserving distributed training, enabling multiple clients to collaboratively train a global model without sharing raw data \cite{kairouz2021advances,chen2021communication,mcmahan2017communication}.
Recent surveys and position papers further discuss the opportunities and challenges of combining FL with LLMs, including system heterogeneity, communication bottlenecks and privacy requirements \cite{xiao2024otafl,yao2024fedllm,cheng2024towardfedllm,wu2025surveyfedft}.
A rapidly growing body of work studies federated fine-tuning of foundation models and LLMs (often called FedLLMs), such as wireless-network oriented fine-tuning of foundation models \cite{wang2024fedftwireless}, federated instruction tuning \cite{zhang2023fedgpt}, mobile-device FwdLLM \cite{Xu2023FwdLLM}, and full-parameter tuning under tight communication constraints \cite{qin2024fulltuning}.

\textbf{Parameter-Efficient and Federated LoRA for LLMs:} PEFT methods, such as adapter modules and LoRA, greatly reduce the number of trainable parameters needed to adapt LLMs \cite{houlsby2019adapter,hu2022lora}.
Many LoRA variants further improve efficiency, expressiveness or robustness, including orthonormal LoRA \cite{buyukakyuz2024olora}, vector-bank \cite{li2024vblora} and Uni-LoRA \cite{li2025unilora}, and FFA-LoRA \cite{sun2024ffalora} for privacy-preserving FL. In the federated setting, several works integrate LoRA into FL for resource-constrained and heterogeneous clients.
Flex-LoRA \cite{bai2024flexlora} proposes a flexible aggregation scheme for clients with heterogeneous computation and data, while FLoRA \cite{wang2024flora} stacks heterogeneous low-rank adapters to improve aggregation robustness.
FedIT \cite{zhang2023fedgpt} and FwdLLM \cite{Xu2023FwdLLM} focus on instruction tuning and billion-parameter models in realistic systems.
More recently, FedTT \cite{ghiasvand2025fedtt} and EcoLoRA \cite{liu2025ecolora} introduce tensorized and energy-aware federated fine-tuning of LLMs. A unified convergence analysis of aggregation–broadcast operators for LoRA-based FL is given in \cite{chen2025aggregationlora}, including a comparison of Sum-of-Products (SoP) and Product-of-Sums (PoS) aggregation. However, \cite{chen2025aggregationlora} does not consider shared vector-bank parameterization or communication-constrained FedLLM settings with quantized updates.
Another line of work studies exact or quantization-aware aggregation for LoRA in FL, including FedEx-LoRA \cite{singhal2024fedexlora}, Quantized LoRA \cite{zhu2024quantlora}, and FedQLoRA \cite{hu2024fedqlora}.
Our method falls into this FedLoRA family, but targets extremely communication-constrained settings and couples LoRA with quantization based encoding of local adapters.

A closely related line mitigates the SoP--PoS bias by alternating optimization or server-side correction of the LoRA factors. \new{RoLoRA~\cite{chen2024rolora} alternates the adapters across rounds, freezing a globally synchronized factor in each round so that averaging the other factor is exact at the product level, with convergence guarantees.} LoRA-FAIR~\cite{bian2025lorafair} introduces a closed-form server-side residual correction for both aggregation bias and initialization drift. LA-LoRA~\cite{liu2026rethinkinglora} alternates $A$ and $B$ at every local step and applies a low-pass Gaussian filter to gradients. These methods operate on raw LoRA adapters.

\textbf{Communication-efficient and Quantized Federated Learning:} Communication-efficient FL has been widely explored via gradient quantization, sparsification and error-compensation.
Representative methods include QSGD \cite{alistarh2017qsgd}, FedPAQ \cite{reisizadeh2020fedpaq}, TernGrad \cite{wen2017terngrad}, signSGD \cite{bernstein2018signsgd} and its majority-vote \cite{bernstein2019signsgd_mv} and error-feedback variants \cite{karimireddy2019errorfeedback}, as well as a recent analysis of error feedback with biased compression in FL \cite{li2023fedef}.
Other communication-efficient schemes compress model updates using binary neural networks or (quantized) compressed sensing techniques \cite{yang2021flbnn,oh2021qcsfl,zheng2024fedcstq,liu2022fedcs}.
Product quantization (PQ) \cite{jegou2011pq} and optimized PQ \cite{ge2013opq} are powerful vector quantizers widely used in approximate nearest-neighbour search, and are natural tools for compressing high-dimensional model updates.
Compared with these works, we focus on quantization and aggregation of LoRA adapters rather than full-precision gradients, and design a PQ-based encoding scheme tailored to federated LoRA for LLMs.

\subsection{Contributions}
Our main contributions are summarized as follows:
\begin{itemize}
     \item We propose the first federated LLM framework utilizing \textit{shared vector-bank parameterization}, which uses disjoint global vector banks to independently generate the low-rank adapter matrices $A$ and $B$.
    \item To resolve the aggregation dilemma, we introduce an \textit{alternating update schedule} where ``A/B-only'' phases ensure perfect aggregation with halved transmission, while ``A+B'' phases utilize a \textit{residual spectral correction} to mitigate the gap between SoP and PoS aggregation. We also provide a comprehensive theoretical analysis of the proposed alternating optimization schedule.
    \item We apply \textit{post-training quantization} to the vector banks to further reduce communication, with client-side \textit{error feedback} compensating for the degradation.
    \item Experiments on Qwen2.5 models demonstrate that our method achieves better performance compared to baselines while reducing communication costs to approximately 1--3\% of standard federated LoRA methods for LLMs.
\end{itemize}

\section{System Model}
\label{sec:system_model}

We consider a wireless FL system consisting of a central server and a set of $K$ distributed clients, indexed by $k \in \mathcal{K} = \{1, \dots, K\}$. The learning process is assumed to take place over $T\geq 1$ communication rounds between the server and the clients.  Each client $k$ possesses a private local set of labeled data $\mathcal{D}_k = \{(\mathbf{x}_i, \mathbf{y}_i)\}_{i=1}^{|\mathcal{D}_k|}$. The objective is to collaboratively fine-tune a pre-trained LLM, parameterized by a weight matrix $\mathcal{W}$, to minimize a global objective function:
\begin{equation}
    \min_{\mathcal{W}} F(\mathcal{W}) \triangleq \sum_{k=1}^K \lambda_k f_k(\mathcal{W}),\label{eq:target}
\end{equation}
where $\lambda_k \ge 0$ is the aggregation weight for client $k$ (typically $\lambda_k = |\mathcal{D}_k| / \sum_{j} |\mathcal{D}_j|$), and $f_k(\mathcal{W}) \triangleq \mathbb{E}_{(\mathbf{x}, \mathbf{y}) \sim \mathcal{D}_k}[\mathcal{L}(\mathcal{W}; \mathbf{x}, \mathbf{y})]$ represents the local empirical risk.

Fine-tuning\footnote{Fine-tuning refers to the training process after pre-training, where a pre-trained model is further adapted to a downstream task or domain using task-specific data.} full LLM parameters in an FL setting is often infeasible due to extreme communication and computation costs. We adopt LoRA, where the pre-trained weight matrix $W_{\ell}^0 \in \mathbb{R}^{d^{\text{out}}_\ell \times d^{\text{in}}_\ell}$ of layer $\ell\in\{1,\dots,L\}$ remains frozen. The update $\Delta W_{\ell}$ is the product of two low-rank matrices:
\begin{equation}
    W_{\ell} = W_{\ell}^0 + \Delta W_{\ell} = W_{\ell}^0 + \frac{\alpha}{r} B_{\ell} A_{\ell}^\top,
    \label{eq:lora_update}
\end{equation}
where $B_{\ell} \in \mathbb{R}^{d^{\text{out}}_\ell \times r}$ and $A_{\ell} \in \mathbb{R}^{d^{\text{in}}_\ell \times r}$ are trainable adapter matrices with rank $r \ll \min(d^{\text{out}}_\ell, d^{\text{in}}_\ell)$, and $\alpha$ is a scaling constant. In the federated setting, each client $k$ maintains its own local adapters $\{A_{\ell,k}, B_{\ell,k}\}_{\ell}$.

\section{Proposed FedFit Framework}
\label{sec:proposed_scheme}

While the standard LoRA formulation in Section \ref{sec:system_model} significantly reduces computational demands compared to full fine-tuning, its deployment in wireless FL faces two critical bottlenecks. First, the communication overhead remains prohibitive for bandwidth-constrained edge networks. Second, the product structure of low-rank matrices induces an aggregation discrepancy under conventional averaging (see Section \ref{sec:aggregation}). To address these challenges, we propose \textbf{FedFit} (see Fig.~1), a communication-efficient framework comprising two integrated components: a disjoint shared vector-bank parameterization, and an alternating optimization schedule.

\begin{figure}[t!]
    \centering
    \includegraphics[width=1.0\linewidth]{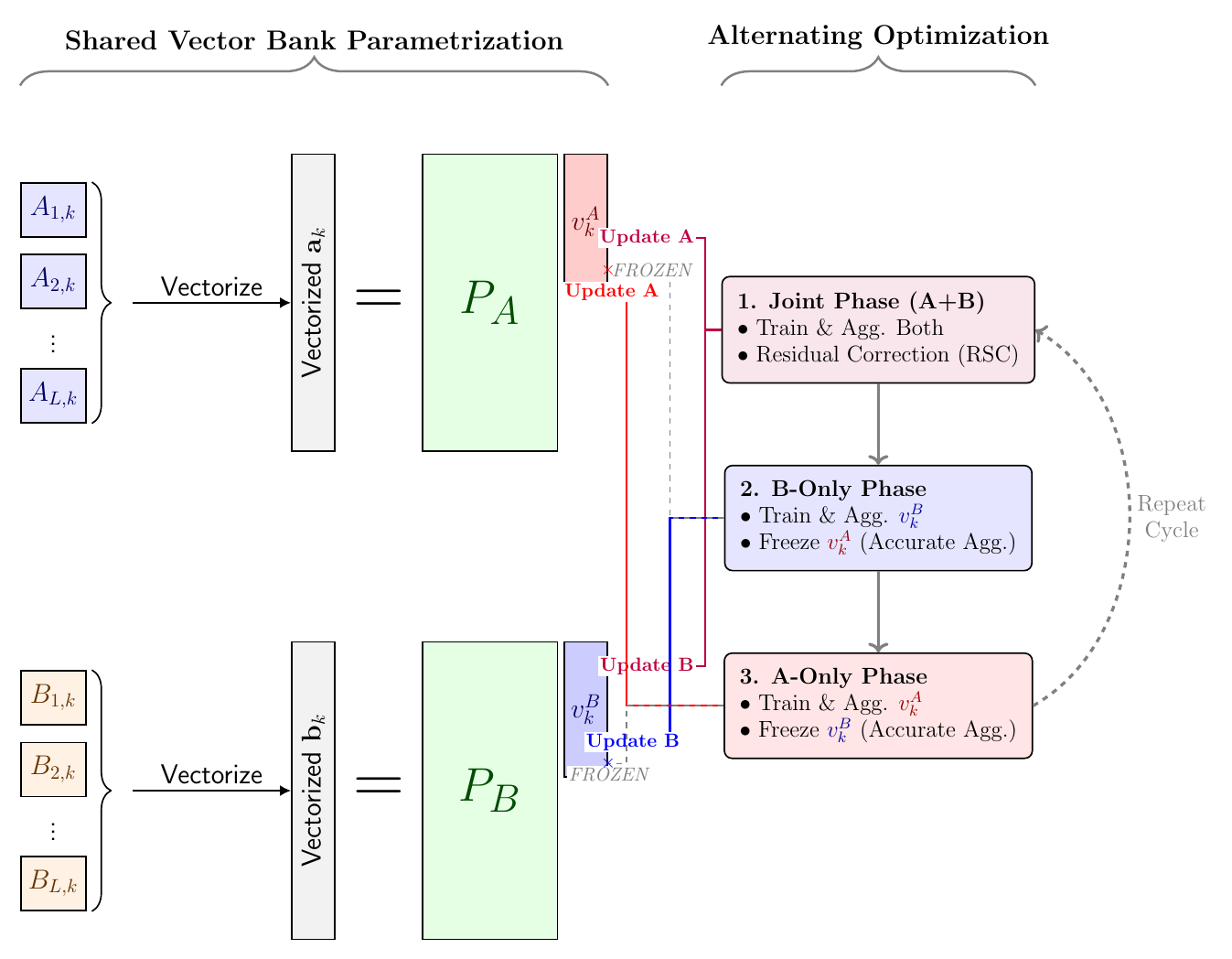}
    \caption{Illustration of client $k$ in FedFit framework}
    \label{fig:framework}
\end{figure}

\subsection{Disjoint Shared Vector-Bank Parameterization}
\label{subsec:vector_Bank}
Although standard LoRA significantly reduces the number of trainable parameters compared to full fine-tuning, the communication overhead remains prohibitive for FL over bandwidth-constrained wireless networks. For instance, fine-tuning a 7B model (e.g., Llama-2) with LoRA still requires transmitting approximately 100-160 million parameters per round. Inspired by recent LoRA variants such as Uni-LoRA~\cite{li2025unilora}, which show that adapter weights are highly redundant, we leverage the insight that sharing parameters across layers' adapters can largely reduce computational and communication complexity while preserving performance.

To achieve this in a federated setting while solving the aggregation dilemma, we introduce the \textit{Disjoint Vector Bank} method. We reformulate the high-dimensional LoRA parameter space by projecting it from two disjoint, low-dimensional global vector banks, $\mathbf{v}^A \in \mathbb{R}^{d_A}$ and $\mathbf{v}^B \in \mathbb{R}^{d_B}$, where the subspace dimensions $d_A, d_B$ are orders of magnitude smaller than the original parameter space.

Let $\text{vec}(A_1, \dots, A_L)$ and $\text{vec}(B_1, \dots, B_L)$ denote the flattened concatenations of all $A$ and $B$ adapter matrices across the $L$ layers. The global reconstruction relationship is:
\begin{equation}
    \text{vec}(A_1, \dots, A_L) = P^A \mathbf{v}^A, \quad \text{vec}(B_1, \dots, B_L) = P^B \mathbf{v}^B.
    \label{eq:global_proj}
\end{equation}
Here, $P^A$ and $P^B$ are fixed, sparse projection matrices constructed to be \textit{isometric} (distance-preserving), i.e. $(P^A)^{\top}P^A=I$, $(P^B)^{\top}P^B=I$. This ensures that optimization in the reduced subspace defined by $\mathbf{v}^A$ and $\mathbf{v}^B$ does not distort the geometry of the original loss landscape.

To facilitate layer-wise operations, we decompose the global projection matrices into layer-specific blocks. We define $P_A$ as the vertical concatenation of layer-specific projection sub-matrices:
\begin{equation}
    P^A = \begin{bmatrix} P^A_1 \\ \vdots \\ P^A_L \end{bmatrix}, \quad P^B = \begin{bmatrix} P^B_1 \\ \vdots \\ P^B_L \end{bmatrix}.
\end{equation}
Consequently, the explicit expression for the adapter parameters of a specific layer $\ell$ is given by:
\begin{equation}
    \text{vec}(A_\ell) = P^A_\ell \mathbf{v}^A, \quad \text{vec}(B_\ell) = P^B_\ell \mathbf{v}^B.
    \label{eq:layer_proj}
\end{equation}
Under this formulation, the learnable parameters transmitted between client and server are strictly $\mathbf{v}^A$ and $\mathbf{v}^B$. The matrices $P^A_\ell$ and $P^B_\ell$ serve as fixed ``gather'' operators shared by all clients, effectively allowing the small number of parameters in the vector banks to be shared globally across all network layers. We use a sparse, random projection as in~\cite{li2025unilora}: 1) Each row of $P_A$ has exactly one nonzero entry; 2) The column index is sampled uniformly from $\{1,\dots,d_A\}$ for adapter matrix $A$; 3) Column-wise normalization is applied.
Formally, we have
\[
(P^A)_{i,j} =
\begin{cases}
\frac{1}{\sqrt{n_j}}, & \text{if row $i$ maps to column $j$}, \\
0, & \text{otherwise},
\end{cases}
\]
where $n_j$ is the number of rows assigned to column $j$. Thus, all LoRA parameters are grouped into $d_A$ shared scalars. Same strategy applies to projection matrix \(P^B\). Because the shared projections are fixed and seeded identically across clients, a client-specific data shift can only perturb the two banks. This acts as a structural regularizer against drift: no single client can unilaterally move the projections even under non-IID scenarios. Compute-heterogeneous clients can train on a strict sub-bank the server zero-pads before aggregation.

\subsection{Alternating Optimization Strategy}
\label{subsec:alt_opt}

Directly optimizing both banks simultaneously induces a bilinear coupling in the weight updates. This coupling not only doubles the communication overhead but also leads to aggregation errors when averaging local updates. To mitigate these issues, we introduce a coordinate-descent training schedule governed by a phase indicator $\phi_t \in \{\text{A-only}, \text{B-only}, \text{A+B}\}$. This approach allows us to selectively freeze parameters to achieve exact aggregation and reduced complexity, or update jointly to ensure convergence.

At each round $t\in\{1,\dots,T\}$, the server broadcasts the current global banks $(\mathbf{v}_t^A, \mathbf{v}_t^B)$ to the selected clients. The clients then perform local training updates restricted to the coordinates specified by $\phi_t$:

\begin{itemize}
    \item \textbf{A-only phase}: Clients update only $\mathbf{v}^A$ while freezing $\mathbf{v}^B$.
    \begin{equation}
        \mathbf{v}_{t+1,k}^A = \mathbf{v}_t^A - \eta_t \, \widehat{\nabla}_{\mathbf{v}^A} {f_k}, \quad \mathbf{v}_{t+1,k}^B = \mathbf{v}_t^B.
        \label{eq:A-phase}
    \end{equation}
    Here, $\widehat{\nabla}_{(\cdot)}{f_k}$ denotes the stochastic gradient computed by client $k$ using its local mini-batches. By freezing $\mathbf{v}^B$, we linearize the weight update with respect to $\mathbf{v}^A$. This not only reduces the uplink communication cost by half but also ensures that the aggregation of local updates is mathematically exact, eliminating approximation errors.

    \item \textbf{B-only phase}: Clients update only $\mathbf{v}^B$ while freezing $\mathbf{v}^A$.
    \begin{equation}
        \mathbf{v}_{t+1,k}^A = \mathbf{v}_t^A, \quad \mathbf{v}_{t+1,k}^B = \mathbf{v}_t^B- \eta_t \, \widehat{\nabla}_{\mathbf{v}^B} {f_k}.
        \label{eq:B-phase}
    \end{equation}
    Symmetric to the A-only phase, this step optimizes the $B$ components with exact aggregation and reduced complexity.

    \item \textbf{A+B phase}: Clients update both banks simultaneously.
    \begin{equation}
        \mathbf{v}_{t+1,k}^A = \mathbf{v}_t^A - \eta_t \, \widehat{\nabla}_{\mathbf{v}^A}{f_k},\quad
        \mathbf{v}_{t+1,k}^B = \mathbf{v}_t^B - \eta_t \, \widehat{\nabla}_{\mathbf{v}^B}{f_k}.
        \label{eq:both-phase}
    \end{equation}
    While solo phases offer efficiency and exact aggregation, less parameters are trained in the local client. This joint phase allows the model to traverse the full parameter space and capture the correlations between matrices $A$ and $B$, even though it incurs slightly higher communication costs and potential aggregation errors.
\end{itemize}

To balance exact aggregation (in solo phases) with the need for full-space optimization (in both phases), we employ a lightweight meta-schedule that repeats short \emph{meta-cycles}:
\begin{center}
\texttt{[both] $\rightarrow$ [A(B)-only] $\rightarrow$ [B(A)-only]}
\end{center}

\section{Aggregation approaches}
\label{sec:aggregation}
The alternating optimization schedule introduced in Section \ref{sec:proposed_scheme} is not merely a heuristic for communication costs reduction and training stability, it is also structurally designed to resolve the fundamental conflict arising when aggregating distributed low-rank updates. While standard federated aggregation rules assume linearity in the parameter space, the low-rank factorization in LoRA introduces a bilinear coupling that breaks this assumption. In this section, we formalize this aggregation dilemma, then demonstrate how our decoupled single-bank phases achieve accurate aggregation, and finally present the residual spectral correction mechanism used to rectify aggregation errors during the joint update phases.

\subsection{The Aggregation Dilemma}
A fundamental challenge in Federated LoRA is the aggregation of local updates. The server aims to compute the global update $\Delta W_{\ell}^{\text{global}}$ that best represents the weighted average of client updates. However, the direct sum of rank-$r$ matrices generally results in a high-rank matrix that violates the LoRA structure. This leads to two conflicting aggregation strategies:

\subsubsection{Product of Sums (PoS)}
Standard approaches (e.g., FedIT) average the adapter matrices $A$ and $B$ independently:
\begin{equation}
    \bar{A}_{\ell} = \sum_{k=1}^{K} \lambda_k A_{\ell,k}, \quad \bar{B}_{\ell} = \sum_{k=1}^{K} \lambda_k B_{\ell,k}.
\end{equation}
While efficient, this strategy fails to capture the coupling between $A$ and $B$, leading to an approximation error where generally $(\sum \lambda_k B_k)(\sum \lambda_k A_k)^\top \neq \sum \lambda_k (B_k A_k^\top)$.

\subsubsection{Sum of Products (SoP)}
Alternatively, aggregating the full updates $\Delta W_{\ell}^{\text{SoP}} = \sum_{k=1}^{K} \lambda_k (B_{\ell,k} A_{\ell,k}^\top)$ yields the exact arithmetic mean of the local weight updates but produces a update $\Delta W_{\ell}^{\text{SoP}}$ with rank up to $K \times r$. Re-compressing this high-rank matrix back into rank-$r$ factors (e.g., via singular value decomposition) is computationally prohibitive for large models and introduces truncation errors.
We resolve this dilemma by designing a parameterization and schedule where PoS aggregation is exact or explicitly correctable.

\subsection{Perfect Aggregation in the A/B-only Phase}
\label{subsec:perfect_Agg}

A key advantage of the alternating schedule is that it restores linearity to the aggregation process during the single-bank update phases. We show that during these phases, the computationally efficient PoS aggregation becomes mathematically equivalent to the precise SoP aggregation.

Assume the system is in an \text{A-only phase} at round $t$. By design, the server does not broadcast updates for the vector bank $\mathbf{v}^B$, and clients freeze their local copies. Consequently, all clients share the exact same frozen sub-bank, $\mathbf{v}_{t,k}^B = \mathbf{v}_t^B$ for all $k \in \mathcal{K}$, and do not update it locally. Following the projection defined in \eqref{eq:layer_proj}, the corresponding LoRA $B$-adapter matrices satisfy $B_{\ell,t,k} = B_\ell(\mathbf{v}_t^B) \equiv B_{\ell,t}$ for all $k$, while the trainable $A_{\ell,k}^t$ matrices may differ across clients due to local training on different datasets.

We first examine the SoP aggregation, which represents the ideal global update. Substituting the frozen $B$ matrix, we obtain:
\begin{align}
    \Delta W_{\ell,t}^{\text{SoP}} &= \sum_{k=1}^K \lambda_k \left( B_{\ell,t} (A_{\ell,t,k})^\top \right) 
    = B_{\ell,t} \left( \sum_{k=1}^K \lambda_k (A_{\ell,t,k})^\top \right) \nonumber \\
    &= B_{\ell,t}  \left( \overline{A}_{\ell,t} \right)^\top.
    \label{eq:sop_derive}
\end{align}
Next, we examine the PoS aggregation. Since $B_{\ell,k}^t = B_\ell^t$ is identical for all clients $k$, the aggregated $B$-adapter is:
\begin{equation}
    \overline{B}_{\ell,t} = \sum_{k=1}^K \lambda_k B_{\ell,t,k} = \left( \sum_{k=1}^K \lambda_k \right) B_{\ell,t} = B_{\ell,t},
\end{equation}
where we utilize the standard FedAvg normalization $\sum_{k=1}^K \lambda_k = 1$. The server then constructs the global update by multiplying the averaged adapters:
\begin{equation}
    \Delta W_{\ell,t}^{\text{PoS}} = \overline{B}_{\ell,t} \left( \overline{A}_{\ell,t} \right)^\top = B_{\ell,t} \left( \overline{A}_{\ell,t} \right)^\top.
\end{equation}
Comparing this result with \eqref{eq:sop_derive}, we observe that:
\begin{equation}
    \Delta W_{\ell,t}^{\text{SoP}} = \Delta W_{\ell,t}^{\text{PoS}}, \quad \forall \ell.
    \label{eq:exact_equality}
\end{equation}
Thus, in A-only phase, the communication-efficient PoS strategy incurs zero aggregation error. By symmetry, the same equality holds for the \text{B-only phase} where matrix $A$ is frozen. This implies that for the majority of the training rounds (the solo phases), FedFit avoids the low-rank approximation errors inherent to standard Federated LoRA.

\subsection{Aggregation Strategy in the A+B phase: From Separate Averaging to Spectral Approximation}

Based on the disjoint vector bank structure defined previously, the adapter matrices $A_{\ell,k}$ and $B_{\ell,k}$ for layer $l$ and client $k$ are derived from the client-specific vector banks $\mathbf{v}_k^A \in \mathbb{R}^{d_A}$ and $\mathbf{v}_k^B \in \mathbb{R}^{d_B}$ via projection matrices. By partitioning the rows of the projection matrix $P_\ell^A$ into $r$ blocks $\{P_{\ell,j}^A\}_{j=1}^r$, and similarly for $P_\ell^B$, we can express the full weight update $\Delta W_{\ell,k} = B_{\ell,k} A_{\ell,k}^\top$ as a sum of rank-1 components,
\begin{equation}
\Delta W_{\ell,k} = \sum_{j=1}^r \left( P_{\ell,j}^B \mathbf{v}_k^B \right) \left( \mathbf{v}_k^A \right)^\top \left( P_{\ell,j}^A \right)^\top.
\end{equation}

Ideally, the server aims to find aggregated global vector banks $\mathbf{v}_\text{agg}^A$ and $\mathbf{v}_\text{agg}^B$ such that the reconstructed global update perfectly matches the average of the clients' local updates, denoted as $\Delta W_{\ell, \text{agg}} = \sum_{k=1}^K \lambda_k\Delta W_{\ell,k}$. However, exact reconstruction is generally impossible because the sum of $K$ rank-1 products cannot be losslessly compressed into a single rank-1 structure for each component $j$. Formally, the inequality
\begin{equation}
\label{eq:inequality}
\frac{1}{K}\sum_{k=1}^K \left( P_{\ell,j}^B \mathbf{v}_k^B \right) \left( \mathbf{v}_k^A \right)^\top \left( P_{\ell,j}^A \right)^\top \neq \left( P_{\ell,j}^B \mathbf{v}_\text{agg}^B \right) \left( \mathbf{v}_\text{agg}^A \right)^\top \left( P_{\ell,j}^A \right)^\top
\end{equation}
holds because the left-hand side represents the true SoP aggregation, which generally possesses a higher rank (at most $K$) within the subspace defined by the projections, whereas the right-hand side restricts the global model to a rank-1 representation. This algebraic constraint forces a choice between two divergent aggregation strategies.

The first strategy focuses on computational simplicity by aggregating the factors $A$ and $B$ separately. Similar as FedAvg, we define the global adapters as the average of local adapters, $\bar{A}_\ell = \sum_k \lambda_kA_{\ell,k}$ and $\bar{B}_\ell = \sum_k \lambda_kB_{\ell,k}$. Due to the linearity of the projection operation where $\text{vec}(A_{\ell,k}) = P_\ell^A \mathbf{v}_k^A$, this adapter-level averaging is mathematically equivalent to averaging the underlying vector banks directly:
\begin{equation}
\mathbf{v}_\text{agg}^A = \sum_{k=1}^K \lambda_k\mathbf{v}_k^A, \quad \mathbf{v}_\text{agg}^B = \sum_{k=1}^K \lambda_k \mathbf{v}_k^B.
\end{equation}
While this method is computationally efficient, it ignores the cross-correlations between $\mathbf{v}_k^A$ and $\mathbf{v}_k^B$.

To capture the coupled structure of the updates, an alternative strategy revisits the goal of approximating the sum of products. Although exact reconstruction is impossible via the layer-wise projections, we can seek $\mathbf{v}_\text{agg}^A$ and $\mathbf{v}_\text{agg}^B$ that minimize the reconstruction error across all layers. The formal optimization objective can be expressed as
{\small\begin{equation}
\min_{\mathbf{v}_\text{agg}^A, \mathbf{v}_\text{agg}^B} \sum_{\ell} \left\| \sum_{j=1}^r P_{\ell,j}^A \left( \sum_{k=1}^K \lambda_k \mathbf{v}_k^A (\mathbf{v}_k^B)^{\top} - \mathbf{v}_\text{agg}^A (\mathbf{v}_\text{agg}^B)^\top \right) (P_{\ell,j}^B)^\top \right\|_F^2.
\end{equation}}

Solving this weighted low-rank approximation problem is intractable due to the layer-specific projection matrices $P^A_{\ell,j}$ and $P^B_{\ell,j}$. We therefore propose a relaxation by removing the projection constraints and minimizing the error directly in the vector bank space. The relaxed problem becomes finding the best rank-1 approximation of the correlation sum
\begin{equation}
\min_{\mathbf{v}_\text{agg}^A, \mathbf{v}_\text{agg}^B} \left\| \sum_{k=1}^K \lambda_k \mathbf{v}_k^A (\mathbf{v}_k^B)^{\top} - \mathbf{v}_\text{agg}^A (\mathbf{v}_\text{agg}^B)^\top \right\|_F^2.
\end{equation}

Even though mathematically this spectral decomposition optimally mitigates the product approximation error, relying solely on this solution poses significant risks to federated training. The obtained rank-1 vectors $\mathbf{v}_\text{agg}^A$ and $\mathbf{v}_\text{agg}^B$, which can be derived from the principal singular vectors of $M=\sum_{k=1}^K \lambda_k\, \mathbf{v}_k^A (\mathbf{v}_k^B)^\top$ can be arbitrarily far from the local client vectors $\mathbf{v}_k^A$ and $\mathbf{v}_k^B$ in the parameter space. Such large deviations destroy the training stability and can lead to divergence when clients resume local training from the aggregated model.

To resolve this conflict, we propose a \textbf{Residual Spectral Correction (RSC)} approach that exploits the advantages of both strategies. We introduce the FedAvg solution $(\bar{\mathbf{v}}^A, \bar{\mathbf{v}}^B)$ as the stable base vectors and add a weighted compensation term to mitigate the product degradation brought by naive averaging. We define the global aggregated banks as:
\begin{equation}
    \mathbf{v}_{\text{agg}}^A = \bar{\mathbf{v}}^A + \gamma \mathbf{r}^A, \quad \mathbf{v}_{\text{agg}}^B = \bar{\mathbf{v}}^B + \gamma \mathbf{r}^B,
\end{equation}
where $\gamma$ is a scalar gain. The correction terms $\mathbf{r}^A, \mathbf{r}^B$ approximate the \textit{residual} product error that FedAvg misses.

To rigorously define the correction terms, we first check the structure of the residual errors. Let $\delta \mathbf{v}_k^A = \mathbf{v}_k^A - \bar{\mathbf{v}}^A$ and $\delta \mathbf{v}_k^B = \mathbf{v}_k^B - \bar{\mathbf{v}}^B$ denote the centered local deviations from the FedAvg anchor. We define the residual matrix $R$ as the weighted sum of the outer products of these deviations:
\begin{equation}
    R \triangleq \sum_{k=1}^K \lambda_k \delta \mathbf{v}_k^A(\delta \mathbf{v}_k^B)^\top = \sum_{k=1}^K \lambda_k (\mathbf{v}_k^A - \bar{\mathbf{v}}^A) (\mathbf{v}_k^B - \bar{\mathbf{v}}^B)^\top.
    \label{eq:residual_def}
\end{equation}
Expanding the term $(\mathbf{v}_k^A - \bar{\mathbf{v}}^A)(\mathbf{v}_k^B - \bar{\mathbf{v}}^B)^\top$ and utilizing the linearity of the summation $\sum \lambda_k (\cdot)$, we establish the relationship between $R$ and the correlation sum $M$:
\begin{align}
    R &= \underbrace{\sum_{k=1}^K \lambda_k \mathbf{v}_k^A (\mathbf{v}_k^B)^\top}_{M} 
       - \underbrace{\left(\sum_{k=1}^K \lambda_k \mathbf{v}_k^A\right)}_{\bar{v}^A} (\bar{\mathbf{v}}^B)^\top \nonumber \\
      &\quad - \bar{\mathbf{v}}^A \underbrace{\left(\sum_{k=1}^K \lambda_k \mathbf{v}_k^B\right)^\top}_{(\bar{\mathbf{v}}^B)^\top} 
       + \underbrace{\left(\sum_{k=1}^K \lambda_k\right)}_{1} \bar{\mathbf{v}}^A (\bar{\mathbf{v}}^B)^\top \nonumber \\
      &= M - \bar{\mathbf{v}}^A (\bar{\mathbf{v}}^B)^\top - \bar{\mathbf{v}}^A (\bar{\mathbf{v}}^B)^\top + \bar{\mathbf{v}}^A (\bar{\mathbf{v}}^B)^\top \nonumber \\
      &= M - \bar{\mathbf{v}}^A (\bar{\mathbf{v}}^B)^\top.
\end{align}
This derivation confirms that $R$ captures exactly the ``missing'' product information, namely, the difference between the true sum of products $M$ and the product of sums provided by FedAvg. Consequently, we obtain $\mathbf{r}^A, \mathbf{r}^B$ by solving the best rank-1 approximation of $R$:
\begin{equation}
    (\mathbf{r}^A, \mathbf{r}^B) = \arg\min_{\mathbf{x},\mathbf{y}} \| R - \mathbf{x} \mathbf{y}^\top \|_F^2.
\end{equation}
By targeting the residual matrix $R$, the correction terms explicitly restore the coupled information lost during averaging, while the FedAvg anchor ensures the global model remains within a stable region of the parameter space.

Directly computing the singular value decomposition (SVD) of $R$ is memory-prohibitive due to the dimension $d_A,d_B \approx 10^6$. However, since $R$ is a weighted sum of $K$ rank-1 terms, its rank is at most $K \ll d_A,d_B$. We utilize a linear-time projected QR decomposition to solve this efficiently. We stack the weighted centered residuals ${\delta} \mathbf{v}_k^A = \sqrt{\lambda_k}(\mathbf{v}_k^A - \bar{\mathbf{v}}^A)$ into tall matrices $\Delta V_A \in \mathbb{R}^{d_A \times K}$ (similarly for $B$). We compute thin QR decompositions $\Delta V_A = Q_A R_A$ and form the small core matrix $C_\delta = R_A R_B^{\top} \in \mathbb{R}^{K \times K}$. The principal singular vectors of $C_\delta$ are then lifted back to $\mathbb{R}^{d_A+d_B}$ to obtain $\mathbf{r}^A$ and $\mathbf{r}^B$ with linear complexity $\mathcal{O}((d_A+d_B)K^2)$. More details can be found in Algo.~1.

\begin{algorithm}[t]
\caption{FedAvg Aggregation with Residual Spectral Correction }
\label{alg:residual_spectral_Agg}
\begin{algorithmic}[1]
\STATE \textbf{Input:} Client vector banks $\{(\mathbf{v}_k^A, \mathbf{v}_k^B, \lambda_k)\}_{k=1}^K$ with $\sum_k \lambda_k = 1$;
residual gain $\gamma \in [0,1]$\new{; gate threshold $\tau = 0.3$}.
\STATE \textbf{Output:} Aggregated banks $\mathbf{v}_{\mathrm{agg}}^A, \mathbf{v}_{\mathrm{agg}}^B$.
\vspace{0.5em}
\STATE \textbf{// Step 1: FedAvg baseline in bank space}
\STATE $ \bar{\mathbf{v}}^A \leftarrow \sum_{k=1}^K \lambda_k \mathbf{v}_k^A,\quad
        \bar{\mathbf{v}}^B \leftarrow \sum_{k=1}^K \lambda_k v_k^B.$
\STATE \textbf{// Step 2: Centered residuals in bank space}
\FOR{$k = 1,\dots,K$}
  \STATE $\delta \mathbf{v}_k^A \leftarrow \mathbf{v}_k^A - \bar{\mathbf{v}}^A, \quad \delta \mathbf{v}_k^B \leftarrow \mathbf{v}_k^B - \bar{\mathbf{v}}^B;$
\ENDFOR
\STATE Stack weighted residuals into tall matrices:
\STATE \quad $ \Delta V_A \leftarrow [\sqrt{\lambda_1}\,\delta \mathbf{v}_1^A,\dots,\sqrt{\lambda_K}\,\delta \mathbf{v}_K^A] \in \mathbb{R}^{d_A \times K}$,
\STATE \quad $ \Delta V_B \leftarrow [\sqrt{\lambda_1}\,\delta \mathbf{v}_1^B,\dots,\sqrt{\lambda_K}\,\delta \mathbf{v}_K^B] \in \mathbb{R}^{d_B \times K}$.
\STATE \textbf{// Step 3: Compute residual correction}
\STATE $[Q_A, R_A] \leftarrow \mathrm{qr}(\Delta V_A)$ \quad \textit{(thin QR; $Q_A \in \mathbb{R}^{d_A \times K}$)}
\STATE $[Q_B, R_B] \leftarrow \mathrm{qr}(\Delta V_B)$ \quad \textit{(thin QR; $Q_B \in \mathbb{R}^{d_B \times K}$)}
\STATE Form small core matrix:
\STATE \quad $\mathcal{C}_{\delta} \leftarrow R_A R_B^\top \in \mathbb{R}^{K \times K}$.
\STATE Compute SVD of the core:
\STATE \quad $[U_{\mathcal{C}}, \Sigma_{\mathcal{C}}, V_{\mathcal{C}}] \leftarrow \mathrm{svd}(\mathcal{C}_{\delta})$.
\STATE Extract singular values and top singular vectors:
\STATE \quad $\sigma_1 \leftarrow \Sigma_{\mathcal{C}}[0,0]$,
\STATE \quad $\mathbf{u}_1 \leftarrow U_{\mathcal{C}}[:,0]$, $\boldsymbol{\nu}_1 \leftarrow V_{\mathcal{C}}[:,0]$.
\STATE \new{\textbf{if} $\rho_1 \triangleq \sigma_1^2 \big/ \sum_{i} \Sigma_{\mathcal{C}}[i,i]^2 < \tau$ \textbf{then return} $\bar{\mathbf{v}}^A, \bar{\mathbf{v}}^B$}
\STATE Scale alignment w.r.t. \(\bar{\mathbf{v}}^A\) and \(\bar{\mathbf{v}}^B\):
\STATE \quad $\alpha \leftarrow \sqrt{\sigma_1 \, \|\bar{\mathbf{v}}^A \| / \|\bar{\mathbf{v}}^B \|}$.
\STATE Define residual correction vectors:
\STATE \quad $\mathbf{r}^A \leftarrow \alpha \,Q_A\mathbf{u}_1, \quad r^B \leftarrow (\sigma_1 / \alpha)\,  Q_B\boldsymbol{\nu}_1;$
\vspace{0.5em}
\STATE \textbf{// Step 4: Combine FedAvg with residual correction}
\STATE $ \mathbf{v}_{\mathrm{agg}}^A \leftarrow \bar{\mathbf{v}}^A + \gamma\, \mathbf{r}^A, \quad \mathbf{v}_{\mathrm{agg}}^B \leftarrow \bar{\mathbf{v}}^B + \gamma\, \mathbf{r}^B; $
\STATE \textbf{return} $\mathbf{v}_{\mathrm{agg}}^A, \mathbf{v}_{\mathrm{agg}}^B$
\end{algorithmic}
\end{algorithm}

\section{Vector Bank Quantization and Communication Cost Analysis}
\subsection{Blockwise Round-To-Nearest (RTN) quantizer}

By projecting the heavy adapter weight matrices into two disjoint vector banks, we can further quantize them before sending them. This is not merely a sequential pipeline of two independent compression stages: the vector-bank parameterization is what makes aggressive low-bit quantization effective, following the classical \emph{transform-coding} pattern of a decorrelating transform followed by a scalar quantizer, e.g., in JPEG and H.264/HEVC. Raw LoRA adapters have heterogeneous per-layer statistics that a uniform blockwise quantizer handles poorly; the bank instead represents the entire transmitted state as a single $1$-D vector whose seeded-projection entries are homogeneous and near-Gaussian, so block-wise RTN with one per-block scale is the natural quantizer. The bank parameterization thus makes quantization \emph{cheap} rather than \emph{lossy}, justifying the two mechanisms as a single co-design rather than an X+Y stack.
Without loss of generality, we consider the joint vector bank $\mathbf{v} = [\mathbf{v}^A; \mathbf{v}^B] \in \mathbb{R}^{d_A + d_B}$ while the communication cost is reduced in half during solo phase. We partition it into $G$ contiguous blocks of length $D$ (assuming $d_A + d_B = G D$ for simplicity):
\begin{align}
\mathbf{v} = \big[\mathbf{v}_{1}^{\top},\dots,\mathbf{v}_{G}^{\top}\big]^{\top}, 
\quad \mathbf{v}_{i} \in \mathbb{R}^{D}, \ i = 1,\dots,G. 
\end{align}

For a given bitwidth $b$ and block size $D$, define the symmetric integer dynamic range \(C(b) \triangleq 2^{b-1} - 1\). On each block $\mathbf{v}_{i}$ we compute a scale
\begin{align}
s_{i}(\mathbf{v}) \triangleq \max_{1 \le j \le D} |\mathbf{v}_{i,j}| + \varepsilon,
\end{align}
with a small $\varepsilon > 0$ to avoid division by zero, and define the blockwise gain \(\alpha_{i}(\mathbf{v}; b) \triangleq \frac{C(b)}{s_{i}(\mathbf{v})}\). The RTN encoder maps $\mathbf{v}$ to integer codes and per-block scales:
\begin{align}
c_{i,j} 
&= \operatorname{clip}\Big( \operatorname{round}\big( \alpha_{i}(\mathbf{v};b)\, \mathbf{v}_{i,j} \big), \, -C(b), \, C(b) \Big), \nonumber \\
s_{i} &= s_{i}(\mathbf{v}),
\end{align}
for $i = 1,\dots,G$ and $j = 1,\dots,D$.
Collecting all $c_{i,j}$ into a tensor $c \in \mathbb{Z}^{L}$ and all scales into
$s = (s_{1},\dots,s_{G})$, we write the quantized representation as $Q(\mathbf{v}; b,D) = \big(c(\mathbf{v}; b,D), s(\mathbf{v}; b,D)\big)$.

The decoder reconstructs $\widehat{\mathbf{v}}$ blockwise as
\begin{align}
\widehat{\mathbf{v}}_{i,j} 
&= \frac{1}{\alpha_{i}(\mathbf{v};b)}\, c_{i,j}
= \frac{s_{i}(\mathbf{v})}{C(b)}\, c_{i,j},
\end{align}
which defines the dequantizer \( \widehat{\mathbf{v}} = \tilde Q\big(Q(\mathbf{v}; b,D)\big)\).

\subsubsection{Decoupled uplink and downlink quantizers}

We allow different quantization configurations for uplink, $(b_{\mathrm{up}}, D_{\mathrm{up}})$, and downlink, $(b_{\mathrm{down}}, D_{\mathrm{down}})$, with possibly different bitwidths and block sizes.

\textbf{Uplink (client $\rightarrow$ server)}: Client $k$ encodes the bank difference $\mathbf{y}_{t,k} = \mathbf{v}_{t,k} - \mathbf{v}_{t-1}$ using $\big(c_{t,k}, s_{t,k}\big) = Q\big(\mathbf{y}_{t,k}; b_{\mathrm{up}}, D_{\mathrm{up}}\big)$ and sends them to the server, which reconstructs $\widehat{\mathbf{y}}_{t,k} = \tilde Q\big(c_{t,k}, s_{t,k}; b_{\mathrm{up}}, D_{\mathrm{up}}\big)$.

\textbf{Downlink (server $\rightarrow$ client)}: After aggregation, the server updates the global bank to $\mathbf{v}_{t+1}$ and may send either: the full bank $\mathbf{v}_{t+1}$, or a downlink vector bank difference $\Delta \mathbf{v}_{t+1} \triangleq \mathbf{v}_{t+1} - \mathbf{v}_{t}$. For concreteness, we suppose the server transmits the difference $\Delta  \mathbf{v}_{t+1}$ using
\begin{equation}
(c_{\mathrm{t+1,down}}, s_{t+1,\mathrm{down}})
= Q\big(\Delta \mathbf{v}_{t+1}; b_{\mathrm{down}}, D_{\mathrm{down}}\big),  
\end{equation}
and each client reconstructs
\begin{equation}
  \Delta \widehat{\mathbf{v}}_{t+1}
= \tilde Q\big(c_{t+1,\mathrm{down}}, s_{t+1,\mathrm{down}}; b_{\mathrm{down}}, D_{\mathrm{down}}\big),
\end{equation}
then updates its local bank via $\mathbf{v}_{t+1,k} = \mathbf{v}_{t,k} + \Delta \widehat{\mathbf{v}}_{t+1}$.

Note that the uplink and downlink quantizers use \emph{independent} parameters $(b_{\mathrm{up}},D_{\mathrm{up}})$ and $(b_{\mathrm{down}},D_{\mathrm{down}})$,
allowing asymmetric compression (e.g., more aggressive compression on uplink, less on downlink).

\subsubsection{Client-side error feedback on the uplink}

To compensate for uplink quantization error, each client $k$ maintains a persistent error accumulator \( \mathbf{e}_{t,k} \in \mathbb{R}^{d_A+d_B},\) initialized as $\mathbf{e}_{0,k} = 0$. At round $t$, instead of quantizing $\mathbf{y}_{t,k}$ directly, the client quantizes
the \emph{error-corrected} quantity $\mathbf{z}_{t,k} \triangleq \mathbf{y}_{t,k} + \mathbf{e}_{t,k}$.
The error-feedback uplink encoder is then
\begin{align}
\big(c_{t,k}, s_{t,k}\big)
&= Q\big(\mathbf{e}_{t,k}; b_{\mathrm{up}}, D_{\mathrm{up}}\big), \nonumber \\
\widehat{\mathbf{e}}_{t,k}
&= \tilde Q\big(c_{t,k}, s_{t,k}; b_{\mathrm{up}}, D_{\mathrm{up}}\big),
\end{align}
with $\widehat{\mathbf{z}}_{t,k}$ reconstructed on the server and interpreted as the quantized update $\widehat{\mathbf{y}}_{t,k} \approx \mathbf{y}_{t,k}$. The client updates its error accumulator after the round:
\begin{equation}
\mathbf{e}_{t+1,k} 
\triangleq \mathbf{z}_{t,k} - \widehat{\mathbf{z}}_{t,k}
= \big(\mathbf{y}_{t,k} + \mathbf{e}_{t,k}\big) - \widehat{\mathbf{y}}_{t,k}.
\end{equation}

Intuitively, $\mathbf{e}_{t+1,k}$ stores the residual quantization error that was \emph{not} transmitted at round $t$.
At the next round, this residual is added back to the new update $\mathbf{y}_{t+1,k}$ before quantization,
which ensures that the long-term sum of transmitted vectors better tracks the sum of exact updates. On the server side, aggregation proceeds with the \emph{decoded} updates $\widehat{\mathbf{y}}_{t,k}$:
\begin{equation}
\mathbf{v}_{t+1}=\mathbf{v}_{t} + \sum_{k=1}^{N} \lambda_{k} \widehat{\mathbf{y}}_{t,k},
\end{equation}
where $\lambda_{k}$ are the FedAvg weights. Error feedback is applied only on the client side; downlink quantization uses its own parameters
$(b_{\mathrm{down}},D_{\mathrm{down}})$ and needs no error accumulator.

\subsection{Compression Ratio and Communication Cost Analysis}

We now quantify the compression ratio of the RTN quantizer applied to a vector bank $\mathbf{v} \in \mathbb{R}^{d_A + d_B}$, assuming a
full-precision baseline of 16-bit brain floating point (BF16) representation. In BF16, the communication cost of transmitting the entire vector $\mathbf{v}$ is \(B_{\text{bf16}} = 16 (d_A+d_B) \ \text{bits}\). We assume that integer codes $c_{i,j}$ are stored using exactly $b$ bits each, each scale $s_i$ is stored in BF16 (16 bits). 

We define the (per-vector) compression ratio $\rho_\text{quant}(b,D)$ as the fraction of the BF16 baseline relative to bits required by RTN:
\begin{align}
\rho_\text{quant}(b,D)
&\triangleq \frac{B_{\text{BF16}}}{B_{\text{RTN}} (b,D)}
= \frac{ 16(d_A+d_B)}{(d_A+d_B)\big(b + \frac{16}{D}\big)} \nonumber\\[0.2em]
&= \frac{16}{\big(b + \frac{16}{D}\big)} \approx \frac{16}{b}.
\end{align}
Thus, for a given bitwidth $b$ and block size $D$, the RTN scheme uses a fraction
$\frac{b}{16}$ of the BF16 bandwidth. For instance, with $b=4$ bits and block size $D = 128$,
\begin{align}
\rho_\text{quant}(4,128)
&= \frac{16}{4+ 16/128}
\approx 3.87
\end{align}
Hence the 4-bit RTN-compressed representation costs about $26\%$ of an BF16 vector, yielding roughly $4\times$ reduction in communication overhead. Before ending this section, two remarks are provided on the communication cost reduction on non-quantization part of proposed method:

\begin{remark}[Vector-bank vs.\ standard LoRA size]
In standard LoRA, each adapted linear layer carries its own pair of dense factors, so the total number of adapter parameters
scales linearly with the number of layers and the
hidden dimension. In our disjoint vector bank design, \emph{all} these factors are generated from two shared vector banks $\mathbf{v}^A \in \mathbb{R}^{d_A}$ and $\mathbf{v}^B \in \mathbb{R}^{d_B}$, with total trainable size $d_A + d_B$ significantly smaller than the total number of weights in adapters. Each element of the vector bank is reused across many adapter entries via the projection matrices $P^A, P^B$ and can be viewed as a representative (or “cluster center’’) for a group of similar values. This leads to a structural compression before any quantization is applied.
\end{remark}

\begin{remark}[Training part scheduler of vector bank]
In our meta-round schedule, each period consists of $m_B$ ``both'' rounds (full $\mathbf{v}^A$ and $\mathbf{v}^B$ sent) followed by
$m_s$ A-only and $m_s$ B-only rounds (one of the two banks). The
resulting meta-round compression ratio is
\begin{equation}\label{eq:rho-scheduler}
\rho_\text{scheduler}
=
\frac{m_B + 2m_s}{m_B +  m_s} \leq 2,
\end{equation}
assuming equal bank length for A and B.
For example, choosing $m_B = 1$ and $m_s = 2$ yields
$\rho_\text{scheduler} = 5/ 3$, i.e.\ about
$40\%$ reduction in adapter communication on top of the vector-bank
parameter quantization.
\end{remark}

\section{Theoretical Analysis}

\new{Having specified the parameterization, the alternating schedule and the quantized uplink and downlink, we now analyze the optimization behavior they induce: how the bank representation reshapes the objective geometry, and whether the alternating updates converge. The analysis assumes exact gradients and full-precision communication. Minibatch sampling and quantization lie outside the proof: under bounded variance and distortion they add error terms that enlarge the convergence neighborhood without changing its qualitative behavior.}

\subsection{Effect of Vector Banks in Convexity}
Here, we discuss the convergence rate of the proposed FL method. We first recall the federated-learning objective in \eqref{eq:target}. From the overview of the LLM, we can reformulate the federated objective as
\begin{equation}
f(\{\Delta W_\ell\}_{\ell=1}^{L})=\sum_{k=1}^N \lambda_k f_k(\{B_{\ell,k}A^\top_{\ell,k}\}_{\ell=1}^L).
\end{equation}
We first make the following definitions and assumptions.

\begin{definition}[\textbf{$\beta$-smoothness}]\label{subsec:smoothness}
For $\beta \geq 0$, a differentiable function $f$ is called $\beta$-smooth if
\begin{equation}
\left\| \nabla f(x) - \nabla f(y) \right\| \leq \beta \left\| x - y \right\|, \quad\forall x,y.
\end{equation}
\end{definition}

\begin{assumption}\label{ass2}
The function $\Delta W_\ell\rightarrow f(\{\Delta W_\ell\}_{\ell=1}^L)$ is $\beta$-smooth.
\end{assumption}

Assumption \ref{ass2} bounds the non-convexity of the LoRA optimization function. Although LLM is highly non-convex, the weak convexity with a relatively small $\mu$ is still easy to satisfy since LoRA is used only for finetuning. From Assumption \ref{ass2}, we have the following theorem, showing the effect of the UniLoRA-like parameter updating scheme adopted in this paper.
\begin{theorem}\label{lem:1}{
Given Assumption \ref{ass2} and \eqref{eq:layer_proj}, suppose the full-space objective is $\beta$-smooth and its gradient is bounded by $G$. Then the induced objective $f(\mathbf v^A,\mathbf v^B)$ is almost block-wise smooth. Specifically, the functions
$\mathbf v^A\mapsto f(\mathbf v^A,\mathbf v^B)$ and
$\mathbf v^B\mapsto f(\mathbf v^A,\mathbf v^B)$ are almost $\beta_A$- and $\beta_B$-smooth, respectively, where
\begin{equation}
\beta_A=\alpha^2\beta r^{-2}
\frac{\max_\ell d^\mathrm{in}_\ell}{\sum_{\ell=1}^L d_\ell^\mathrm{in}},\quad
\beta_B=\alpha^2\beta r^{-2}
\frac{\max_\ell d^\mathrm{out}_\ell}{\sum_{\ell=1}^L d_\ell^\mathrm{out}}.
\end{equation}
Moreover, the induced objective satisfies the following mixed-smoothness property:
\begin{equation}
\begin{aligned}
&\|\nabla_{\mathbf v^B} f(\mathbf v^{A\prime},\mathbf v^B)
-
\nabla_{\mathbf v^B} f(\mathbf v^A,\mathbf v^B)\|_2
\le
\beta_{\rm mix}\|\mathbf v^{A\prime}-\mathbf v^A\|_2,\\
&\|\nabla_{\mathbf v^A} f(\mathbf v^A,\mathbf v^{B\prime})
-
\nabla_{\mathbf v^A} f(\mathbf v^A,\mathbf v^B)\|_2
\le
\beta_{\rm mix}\|\mathbf v^{B\prime}-\mathbf v^B\|_2,
\end{aligned}
\end{equation}
where
\begin{equation}
\beta_{\rm mix}
\lesssim
\frac{\alpha G}{r}
+
\sqrt{\beta_A\beta_B}.
\end{equation}
Here, the word ``almost'' comes from the scaling-law approximation under the random projection in \eqref{eq:layer_proj}.}
\end{theorem}

\begin{proof}
See Appendix I.
\end{proof}
\begin{remark}
By Theorem~\ref{lem:1}, we see that although the vector banks broadly preserve the smoothness and weak convexity properties of the original fine-tuning objective.
However, as $d_A$ and $d_B$ become smaller, the normalization factor in $P^A$ and $P^B$ also become smaller, leading to more intense changes in $\mathbf{v}^A$ and $\mathbf{v}^B$. Thus, the effective smoothness and weak convexity constants of the vector bank objective worsen when $d_A$ and $d_B$ are taken smaller. Since $d_A$ and $d_B$ are directly tied to both computation and communication costs in the system, this induces an explicit trade-off between complexity and performance, which is consistent with both standard intuition and our empirical observations.

On the other hand, this also highlights that, in the systems considered in this paper, the optimization landscape in fact contains a large number of local minima and saddle points. Merely ensuring convergence to some stationary point therefore does not guarantee good performance. In particular, alternating “freeze-and-update’’ schemes (e.g., A-only phase and B-only phase) are especially prone to becoming trapped in such points. This is one of the reasons we prioritize updating $A$ and $B$ jointly whenever possible.
\end{remark}
\subsection{Convergence of the Proposed Method}
\noindent\textbf{\textit{A-Only Phase and B-Only Phase:}}
We first discuss the A-only and B-only phases as defined in Eq. \eqref{eq:A-phase} and Eq. \eqref{eq:B-phase}, respectively. We have the following theorem.
\begin{theorem}\label{the:1}
{Assume that $f$ is lower bounded by $f_\star$.}
If we alternately conduct A-only and B-only phases, the FedFit algorithm can converge to a vector bank $(\mathbf{v}^A, \mathbf{v}^B)$ with $0<\eta_t<2/\max\{\beta_A,\beta_B\}$ such that $\nabla f(\mathbf{v}^A, \mathbf{v}^B)=0$.
{Moreover, after $T$ complete A/B cycles, we have}
\begin{equation}
{
\min_{0\le t<T}\mathcal{G}^{\rm alt}_t
\le
\frac{
f(\mathbf{v}^A_0,\mathbf{v}^B_0)-f_\star
}{
\sum_{t=0}^{T-1}c_t^{\rm alt}
},
}
\end{equation}
{where}
\begin{equation}
{
\mathcal{G}^{\rm alt}_t
=
\|\nabla_{\mathbf{v}^A} f(\mathbf{v}^A_t,\mathbf{v}^B_t)\|_2^2
+
\|\nabla_{\mathbf{v}^B} f(\mathbf{v}^A_{t+1},\mathbf{v}^B_t)\|_2^2,
}
\end{equation}
{and}
\begin{equation}
{
c_t^{\rm alt}
=
\min\left\{
\eta_t\left(1-\frac{\beta_A\eta_t}{2}\right),
\eta_t\left(1-\frac{\beta_B\eta_t}{2}\right)
\right\}>0.
}
\end{equation}
{In particular, if $\eta_t=\eta$ is constant, then}
\begin{equation}
{
\min_{0\le t<T}\mathcal{G}^{\rm alt}_t
\le
\frac{
f(\mathbf{v}^A_0,\mathbf{v}^B_0)-f_\star
}{
T\eta\left(1-\frac{\max\{\beta_A,\beta_B\}\eta}{2}\right)
}
=
\mathcal{O}\left(\frac{1}{T}\right).
}
\end{equation}
\end{theorem}
\begin{proof}
Due to the symmetry, we begin with the A-only phase. 
Following basic gradient descent principles, we have the following one-step descent inequality. Since SoP=PoS, we can directly discuss the corresponding centralized algorithm.
\begin{equation}
\begin{aligned}
f(\mathbf{v}^A_{t+1}, \mathbf{v}^B) 
\le &f(\mathbf{v}^A_{t}, \mathbf{v}^B)
+ \langle \nabla_{\mathbf{v}^A} f(\mathbf{v}^A_{t}, \mathbf{v}^B), \mathbf{v}^A_{t+1}-\mathbf{v}^A_{t} \rangle\\
&+ \frac{\beta_A}{2} \|\mathbf{v}^A_{t+1}-\mathbf{v}^A_{t}\|_2^2.
\end{aligned}
\end{equation}
From the updating function Eq. \eqref{eq:A-phase}, by setting $0<\eta_t<2/\beta_A$, we have
\begin{equation}
f(\mathbf{v}^A_{t+1}, \mathbf{v}^B)-f(\mathbf{v}^A_{t}, \mathbf{v}^B)
\le \frac{\beta_A\eta_t-2}{2}\eta_t \|\nabla_{\mathbf{v}^A} f(\mathbf{v}^A_{t}, \mathbf{v}^B)\|_2^2.
\end{equation}
Thus, the A-only gradient updating converges to some point where $\nabla_{\mathbf{v}^A} f(\mathbf{v}^A, \mathbf{v}^B)=0$. Symmetrically, we have the one-step descent inequality for B-only phase by
\begin{equation}
f(\mathbf{v}^A, \mathbf{v}^B_{t+1})-f(\mathbf{v}^A, \mathbf{v}^B_{t})
\le \frac{\beta_B\eta_t-2}{2}\eta_t \|\nabla_{\mathbf{v}^B} f(\mathbf{v}^A, \mathbf{v}_{t}^B)\|_2^2.
\end{equation}
and similarly, the B-only gradient updating converges to some point where $\nabla_{\mathbf{v}^B} f(\mathbf{v}^A, \mathbf{v}^B)=0$.

{
Moreover, combining the descent inequalities of one A-only step and the subsequent B-only step, we have
}
\begin{equation}
{
f(\mathbf{v}^A_{t+1},\mathbf{v}^B_{t+1})
-
f(\mathbf{v}^A_t,\mathbf{v}^B_t)
\le
-c_t^{\rm alt}\mathcal{G}^{\rm alt}_t,
}
\end{equation}
{
where $\mathcal{G}^{\rm alt}_t$ and $c_t^{\rm alt}$ are defined above. Summing this inequality from $t=0$ to $T-1$ gives
}
\begin{equation}
{
\sum_{t=0}^{T-1}c_t^{\rm alt}\mathcal{G}^{\rm alt}_t
\le
f(\mathbf{v}^A_0,\mathbf{v}^B_0)
-
f(\mathbf{v}^A_T,\mathbf{v}^B_T)
\le
f(\mathbf{v}^A_0,\mathbf{v}^B_0)-f_\star.
}
\end{equation}
{
Therefore,
}
\begin{equation}
{
\min_{0\le t<T}\mathcal{G}^{\rm alt}_t
\le
\frac{
f(\mathbf{v}^A_0,\mathbf{v}^B_0)-f_\star
}{
\sum_{t=0}^{T-1}c_t^{\rm alt}
}.
}
\end{equation}
{
When $\eta_t=\eta$, this further yields the rate $\mathcal{O}(1/T)$. Therefore, by iterating between A-only and B-only phases, the FedFit algorithm can converge to some $(\mathbf{v}^A,\mathbf{v}^B)$ such that $\nabla f(\mathbf{v}^A,\mathbf{v}^B)=0$, which finishes the proof.
}
\end{proof}

Theorem \ref{the:1} shows that simply iterating between the two separately updating phases is enough for convergence, which is similar to the conclusion in \cite{wang2025adflora}.

Finally we discuss the A+B phase. The following theorem shows that if we choose some reasonable learning rate $\eta_t$. Here, the difference mainly lies in the bilinear-like LoRA function.

\begin{theorem}\label{the:2}
{
Assume that $f$ is lower bounded by $f_\star$ and satisfies the block-wise smoothness and mixed-smoothness properties in Theorem \ref{lem:1}.
}
If we only conduct the A+B phase, the FedFit algorithm can still converge to a vector bank $(\mathbf{v}^A,\mathbf{v}^B)$ such that $\nabla f(\mathbf{v}^A,\mathbf{v}^B)=0$ with sufficiently small $\eta_t$.
{
Moreover, after $T$ iterations, we have
}
\begin{equation}
{
\min_{0\le t<T}
\|\nabla f(\mathbf v_t^A,\mathbf v_t^B)\|_2^2
\le
\frac{
f(\mathbf v_0^A,\mathbf v_0^B)-f_\star
}{
\sum_{t=0}^{T-1} c_t^{\rm AB}
},
}
\end{equation}
{
where
}
\begin{equation}
{
c_t^{\rm AB}
=
\eta_t\left(1-\frac{(\min\{\beta_A,\beta_B\}+\beta_{\rm mix})\eta_t}{2}\right)>0.
}
\end{equation}
{
In particular, if $\eta_t=\eta$ is constant and
$0<\eta<2/(\max\{\beta_A,\beta_B\}+\beta_{\rm mix})$, then
}
\begin{equation}\footnotesize
{
\min_{0\le t<T}
\|\nabla f(\mathbf v_t^A,\mathbf v_t^B)\|_2^2
\le
\frac{
f(\mathbf v_0^A,\mathbf v_0^B)-f_\star
}{
T\eta\left(1-\frac{(\max\{\beta_A,\beta_B\}+\beta_{\rm mix})\eta}{2}\right)
}
=
\mathcal O\left(\frac{1}{T}\right).
}
\end{equation}
\end{theorem}
\begin{proof}
See Appendix \ref{proof:2}.
\end{proof}

\begin{remark}
The above method shows that the A+B method also converges to the same suboptimal set regardless of the fact that SoP$\neq$PoS. However, it requires smaller learning rate and frequent aggregations, which are hard to satisfy in practice. Thus, simply relying on the A+B phase is still likely to cause training failure.
\end{remark}

\section{Numerical Results}

\textbf{Comparison with SOTA FL methods for LLMs:} We start by comparing proposed methods with existing LoRA based FL methods for LLM. We select Qwen2.5-{0.5B/1.5B/3B}-Instruct \cite{qwen2025qwen25technicalreport} as our base models and perform PEFT on them with general-purpose supervised fine-tuning\footnote{ Here ``B" means one billion number of parameters.}. The LoRA rank is set to \(r=4\) with scaling constant \(\alpha =8\) unless specifying explicitly. For proposed method, we set the vector bank length for adapter matrices \(A,B\) to \(128 \times 2048 = 262144\) which is significantly smaller than the number of adapter parameters.   For all experiments, the maximum learning rate is \(5 \times 10^{-5}\) decayed to \(1 \times 10^{-5}\) with a cosine learning rate scheduling. Each FL round consists of \(100\) training steps unless specifying explicitly. All training takes \(T=20\) FL rounds until the decay of evaluation perplexity is negligible. We split the Dolly dataset, a high-quality instructions dataset \cite{DatabricksBlog2023DollyV2}, into 90\% for training and 10\% for evaluation. Unless noted, ablations in this section (Tabs.~V--XIII) use $K = 4$ on Qwen2.5-1.5B-Instruct at the canonical bank dimension $L = 524{,}288$. \new{All runs share one fixed seed and data partition, so comparisons are paired; repeating the $K = 4$ Qwen2.5-1.5B cell at three seeds moves the $7$-task average by $0.10$~pt.}

\begin{table*}[t!]
\centering
\small

\renewcommand{\arraystretch}{1.0}
\begin{tabular}{l l c c c c}
\toprule
Model & Method & Communication Cost & Compression Ratio & Eval Perplexity & Average Score \\
\midrule
\multirow{10}{*}{\text{Qwen2.5-0.5B-IT}}
  & FedIT~\cite{zhang2023fedgpt}            & 4.40M  & 1.00   & 7.42 & 50.48 \\
  & Flex-LoRA~\cite{bai2024flexlora}        & 4.40M  & 1.00   & 7.41 & 50.18 \\
  & FFA-LoRA~\cite{sun2024ffalora}         & 2.20M  & 2.00   & 7.65 & 50.49 \\
  & RoLoRA~\cite{chen2024rolora}         & 2.22M & 2.25 & 7.79 & \textbf{50.78} \\
  & Quantized-LoRA~\cite{zhu2024quantlora} & 1.16M & 4.29 & 7.64 & 50.63 \\
  & LoRA-FAIR~\cite{bian2025lorafair}      & 4.43M & 1.13 & 7.56 & 50.57 \\
  & LA-LoRA~\cite{liu2026rethinkinglora}        & 4.40M & 1.00 & 8.60 & 50.66 \\
\cmidrule(l){2-6}
  & FedFit          & 0.70M  & \underline{6.30}  & \underline{7.30}  & \underline{50.69} \\
  & FedFit (RTN4)   & 0.19M  & \textbf{23.69}    & \textbf{7.29}     & 50.37 \\
\midrule
\multirow{10}{*}{\text{Qwen2.5-1.5B-IT}}
  & FedIT~\cite{zhang2023fedgpt}            & 9.23M   & 1.00            & 5.68             & \underline{60.69} \\
  & Flex-LoRA~\cite{bai2024flexlora}        & 14.97M  & 1.00            & 5.66             & 60.48 \\
  & FFA-LoRA~\cite{sun2024ffalora}         & 7.48M   & 1.00            & 5.81             & 60.23 \\
  & RoLoRA~\cite{chen2024rolora}         & 4.63M & 2.07 & 5.84 & 60.40 \\
  & Quantized-LoRA~\cite{zhu2024quantlora} & 2.38M & 4.03 & 5.74 & 60.48 \\
  & LoRA-FAIR~\cite{bian2025lorafair}      & 9.27M & 1.03 & 5.71 & 60.54 \\
  & LA-LoRA~\cite{liu2026rethinkinglora}        & 9.23M & 1.00 & 6.57 & 59.95 \\
\cmidrule(l){2-6}
  & FedFit          & 0.70M   & \underline{13.21} & \textbf{5.58}   & \textbf{60.84} \\
  & FedFit (RTN4)   & 0.19M   & \textbf{49.72}    & \underline{5.59} & \underline{60.69} \\
\midrule
\multirow{10}{*}{\text{Qwen2.5-3B-IT}}
  & FedIT~\cite{zhang2023fedgpt}            & 14.97M  & 1.00            & 5.10            & 63.63 \\
  & Flex-LoRA~\cite{bai2024flexlora}        & 14.97M  & 1.00            & 5.11            & 63.30 \\
  & FFA-LoRA~\cite{sun2024ffalora}         & 7.48M   & 2.00            & 5.35            & 63.17 \\
  & RoLoRA~\cite{chen2024rolora}         & 7.51M  & 1.99 & 5.37 & 63.36 \\
  & Quantized-LoRA~\cite{zhu2024quantlora} & 3.84M  & 3.90 & 5.19 & 63.42 \\
  & LoRA-FAIR~\cite{bian2025lorafair}      & 15.01M & 1.00 & 5.13 & \underline{63.65} \\
  & LA-LoRA~\cite{liu2026rethinkinglora}        & 14.97M & 1.00 & 5.90 & 62.90 \\
\cmidrule(l){2-6}
  & FedFit         & 0.70M    & \underline{21.41} & \textbf{5.03} & \textbf{63.76} \\
  & FedFit (RTN4)  & 0.19M    & \textbf{80.60}    & \textbf{5.03} & 63.55 \\
\midrule
\multirow{10}{*}{\text{Qwen2.5-7B-IT}}
  & FedIT~\cite{zhang2023fedgpt}          & 20.22M & 1.00            & \textbf{4.52} & 67.56 \\
  & Flex-LoRA~\cite{bai2024flexlora}      & 20.22M & 1.00            & \textbf{4.52} & 67.47 \\
  & FFA-LoRA~\cite{sun2024ffalora}       & 11.14M & 1.81            & 4.63          & 67.26 \\
  & RoLoRA~\cite{chen2024rolora}         & 10.11M & 2.00            & 4.58          & 67.42 \\
  & Quantized-LoRA~\cite{zhu2024quantlora} & 5.12M  & 3.95            & 4.53          & 67.45 \\
  & LoRA-FAIR~\cite{bian2025lorafair}      & 20.22M & 1.00            & \textbf{4.52} & 67.69 \\
  & LA-LoRA~\cite{liu2026rethinkinglora}        & 20.22M & 1.00            & 4.73          & 67.00 \\
\cmidrule(l){2-6}
  & FedFit        & 0.70M & \underline{28.89} & \textbf{4.52} & \underline{67.71} \\
  & FedFit (RTN4) & 0.19M & \textbf{106.42}   & \textbf{4.52} & \textbf{68.01} \\
\bottomrule
\end{tabular}
\caption{Comparison of FL methods with LoRA rank \(r=4\) with 4 clients and 20 FL rounds for Qwen2.5-{0.5B/1.5B/3B/7B}-Instruct models. Communication cost denotes the number of bytes per client per direction (downlink/uplink) assuming BF16 representation (2 bytes/weight). Compression ratio is based on FedIT with full precision. RTN4 stands for 4-bit round-to-nearest quantization. Proposed method achieves slightly better performance with a large compression ratio of \textbf{20-100$\times$} depending on the model size compared to SOTA methods.}
\label{tab:communication_cost_And_Bench_score}
\end{table*}

\begin{table*}[t!]
\centering
\small

\label{tab:peft-comparison}
\renewcommand{\arraystretch}{1.0}
\begin{tabular}{l l  c c c c c c c c}
\toprule
Model & Method  & BoolQ & PIQA  & HellaS. & WinoG. & ARC-e & ARC-c & OBQA & Average Score\\
\midrule

\multirow{9}{*}{\textbf{Qwen2.5-0.5B-IT}}
& FedIT~\cite{zhang2023fedgpt}       & 67.92 & 69.97 & 40.00 & 56.43 & 65.15 & 29.10 & 24.80 & 50.48 \\

& FFA-LoRA~\cite{sun2024ffalora}    & 67.89 & 70.35 & 40.03 & 55.96 & 65.28 & 29.35 & 24.60 & 50.49\\

& Flex-LoRA~\cite{bai2024flexlora}   & 68.50 & 69.97 & 39.82 & 55.88 & 64.60 & 28.50 & 24.00 & 50.18 \\

& RoLoRA~\cite{chen2024rolora}          & 68.01 & 70.89 & 40.25 & 56.83 & 65.66 & 30.03 & 23.80 & \textbf{50.78} \\

& Quantized-LoRA~\cite{zhu2024quantlora}  & 67.83 & 70.40 & 40.06 & 56.51 & 65.40 & 29.61 & 24.60 & 50.63 \\

& LoRA-FAIR~\cite{bian2025lorafair}       & 68.20 & 70.46 & 40.00 & 56.04 & 65.61 & 29.10 & 24.60 & 50.57 \\

& LA-LoRA~\cite{liu2026rethinkinglora}         & 68.13 & 70.46 & 40.51 & 56.04 & 65.87 & 30.03 & 23.60 & 50.66 \\

\cmidrule(l){2-10}

& \textbf{FedFit (RTN4)} & 67.25 & 70.29 & 39.62 & 55.56 & 65.49 & 29.35 & 25.00 & 50.37 \\

& \textbf{FedFit}  & 67.92 & 70.29 & 39.66 & 56.12 & 66.20 & 29.86 & 24.80 & \underline{50.69} \\

\midrule

\multirow{9}{*}{\textbf{Qwen2.5-1.5B-IT}}
& FedIT~\cite{zhang2023fedgpt}       & 79.45 & 75.41 & 50.72 & 64.17 & 78.03 & 44.45 & 32.60 & \underline{60.69} \\

& FFA-LoRA~\cite{sun2024ffalora}    & 79.24 & 75.63 & 50.95 & 63.22 & 77.40 & 43.60 & 31.60 & 60.23 \\

& Flex-LoRA~\cite{bai2024flexlora}   & 79.63 & 75.84 & 50.74 & 63.38 & 77.48 & 44.45 & 31.80 & 60.48 \\

& RoLoRA~\cite{chen2024rolora}          & 79.33 & 75.90 & 50.74 & 63.46 & 77.23 & 43.77 & 32.40 & 60.40 \\

& Quantized-LoRA~\cite{zhu2024quantlora}  & 79.57 & 75.84 & 50.80 & 63.22 & 77.57 & 45.14 & 31.20 & 60.48 \\

& LoRA-FAIR~\cite{bian2025lorafair}       & 79.82 & 75.57 & 50.87 & 63.22 & 77.99 & 44.54 & 31.80 & 60.54 \\

& LA-LoRA~\cite{liu2026rethinkinglora}         & 78.13 & 75.95 & 50.81 & 63.38 & 76.30 & 43.09 & 32.00 & 59.95 \\

\cmidrule(l){2-10}

& \textbf{FedFit (RTN4)}  & 80.28 & 75.57 & 50.47 & 63.77 & 78.75 & 44.37 & 31.60 & \underline{60.69} \\

& \textbf{FedFit}  & 80.61 & 75.30 & 50.44 & 65.04 & 77.82 & 44.71 & 32.00 & \textbf{60.84} \\

\midrule

\multirow{9}{*}{\textbf{Qwen2.5-3B-IT}}
& FedIT~\cite{zhang2023fedgpt}       & 82.94 & 78.29 & 55.41 & 69.46 & 79.34 & 47.95 & 32.00 & 63.63 \\

& FFA-LoRA~\cite{sun2024ffalora}    & 81.13 & 77.75 & 56.00 & 70.09 & 78.91 & 46.33 & 32.00 & 63.17 \\

& Flex-LoRA~\cite{bai2024flexlora}   & 82.69 & 77.91 & 55.44 & 69.14 & 79.34 & 47.18 & 31.40 & 63.30 \\

& RoLoRA~\cite{chen2024rolora}          & 81.44 & 78.18 & 55.66 & 69.85 & 79.38 & 46.84 & 32.20 & 63.36 \\

& Quantized-LoRA~\cite{zhu2024quantlora}  & 82.20 & 78.13 & 55.30 & 69.61 & 79.34 & 47.35 & 32.00 & 63.42 \\

& LoRA-FAIR~\cite{bian2025lorafair}       & 82.75 & 78.56 & 55.35 & 69.93 & 79.46 & 47.53 & 32.00 & \underline{63.65} \\

& LA-LoRA~\cite{liu2026rethinkinglora}         & 80.12 & 78.18 & 56.13 & 69.53 & 78.11 & 46.42 & 31.80 & 62.90 \\

\cmidrule(l){2-10}

& \textbf{FedFit (RTN4)}  & 84.13 & 78.02 & 54.60 & 69.38 & 79.80 & 47.10 & 31.80 & 63.55 \\

& \textbf{FedFit }  & 83.27 & 78.29 & 54.73 & 68.59 & 80.51 & 48.12 & 32.80 & \textbf{63.76} \\

\midrule

\multirow{9}{*}{\textbf{Qwen2.5-7B-IT}}
& FedIT~\cite{zhang2023fedgpt}       & 86.85 & 79.60 & 60.66 & 72.22 & 83.80 & 56.40 & 33.40 & 67.56 \\

& FFA-LoRA~\cite{sun2024ffalora}    & 86.73 & 80.52 & 61.31 & 72.38 & 82.37 & 54.10 & 33.40 & 67.26 \\

& Flex-LoRA~\cite{bai2024flexlora}   & 86.91 & 79.76 & 60.70 & 72.22 & 83.71 & 55.97 & 33.00 & 67.47 \\

& RoLoRA~\cite{chen2024rolora}          & 87.22 & 80.09 & 60.71 & 71.90 & 83.38 & 55.03 & 33.60 & 67.42 \\

& Quantized-LoRA~\cite{zhu2024quantlora}  & 86.91 & 79.82 & 60.75 & 71.98 & 83.50 & 55.80 & 33.40 & 67.45 \\

& LoRA-FAIR~\cite{bian2025lorafair}       & 87.03 & 80.20 & 60.59 & 72.69 & 83.75 & 56.57 & 33.00 & 67.69 \\

& LA-LoRA~\cite{liu2026rethinkinglora}         & 86.51 & 80.09 & 61.46 & 71.82 & 81.73 & 53.16 & 34.20 & 67.00 \\

\cmidrule(l){2-10}

& \textbf{FedFit (RTN4)} & 87.25 & 80.47 & 60.60 & 73.56 & 84.60 & 55.97 & 33.60 & \textbf{68.01} \\

& \textbf{FedFit}        & 86.64 & 80.09 & 60.25 & 72.93 & 84.39 & 56.06 & 33.60 & \underline{67.71} \\

\bottomrule
\end{tabular}
\caption{Accuracy comparison of FL methods applied to Qwen2.5-{0.5B/1.5B/3B/7B}-Instruct models across 7 common-sense reasoning benchmarks. Proposed method outperforms all baselines with full precision and achieves comparable performances when 4-bit RTN quantization is applied during weight transmission.}
\label{tab:benchmark_score_per_dataset}
\end{table*}

We compare the proposed method with popular state-of-the-art FL baselines, including FedIT~\cite{zhang2023fedgpt}, Flex-LoRA~\cite{bai2024flexlora}, FFA-LoRA~\cite{sun2024ffalora}, RoLoRA~\cite{chen2024rolora}, Quantized-LoRA~\cite{zhu2024quantlora}, LoRA-FAIR~\cite{bian2025lorafair}, and the non-DP variant of LA-LoRA~\cite{liu2026rethinkinglora} for LLM federated finetuning. In Tab.~\ref{tab:communication_cost_And_Bench_score}, we report performance on common-sense reasoning benchmarks including BoolQ, PIQA, HellaSwag, Winogrande, ARC-easy, ARC-challenge, and OBQA in \cite{hu2023llm}, together with the communication cost in BF16 representation per weight, the corresponding compression ratio, and the evaluation perplexity. Across LLMs of different sizes, the proposed method consistently achieves lower evaluation perplexity and higher average scores, while substantially reducing communication cost. For example, for Qwen2.5-7B-Instruct, our method requires only about \(190\)~KB of communicated parameters, compared to \(20\)~MB for standard LoRA-based methods such as FedIT and Flex-LoRA, yielding a compression ratio of roughly \(100\times\). Detailed benchmark scores for each dataset and method are provided in Tab.~\ref{tab:benchmark_score_per_dataset}.

We further compare the training dynamics of different FL methods. As illustrated in Fig.~\ref{fig:qwen2.5_0.5b_it_fl_methods_eval_loss}, we plot the evaluation perplexity on Qwen2.5-0.5B-Instruct. One can observe that the proposed algorithm consistently outperforms all baselines throughout training. In particular, Flex-LoRA and FedIT plateau after the \(15\)th FL round, whereas the evaluation perplexity of the proposed method continues to decrease.

\begin{figure}[t!]
  \centering
  \begin{tikzpicture}
    \begin{axis}[
      width=0.3\textwidth,    
      height=0.25\textwidth,
      scale only axis,
      xmin=0.5, xmax=20.5,
      ymin=7.25, ymax=8.45,
      xlabel={FL Round},
      ylabel={Eval Perplexity},
      xtick={2,4,...,20},
      ytick={7.4,7.6,7.8,8.0,8.2,8.4},
      grid=both,
      grid style={line width=.1pt, draw=gray!30},
      major grid style={line width=.2pt,draw=gray!50},
      legend style={at={(0.98,0.98)}, anchor=north east, font=\scriptsize, draw=gray},
      tick label style={font=\footnotesize},
      label style={font=\footnotesize},
      clip marker paths=true,
    ]

    \addplot[
      color=brown, mark=x, thick, mark size=3, line width=1.2pt
    ] coordinates {
   (1,8.398146176637388)(2,8.132272900963004)(3,8.010697881219775)(4,7.936621374983175)(5,7.876964948223221)(6,7.826148979389357)(7,7.786502355794833)(8,7.763353526238865)(9,7.74612572536553)(10,7.723522292136398)(11,7.707297933549379)(12,7.697935728099033)(13,7.688269608111742)(14,7.677800996404267)(15,7.676898597950161)(16,7.671552279270188)(17,7.673590101100847)(18,7.672664416556222)
(19,7.671497408248604)(20,7.67175347637381)
    };
    \addlegendentry{FFA-LoRA \cite{sun2024ffalora}}

    \addplot[
      color=blue, mark=o, thick, mark size=1.5, line width=1.2pt
    ] coordinates {
      (1,8.163025388260055)(2,7.85469693564566)(3,7.743544300831668)(4,7.660240694618889)(5,7.6005937981812846)(6,7.560770179577093)(7,7.530043895999796)(8,7.49412147471013)(9,7.475300424831444)(10,7.461708877754723)(11,7.446187170573194)(12,7.436088963248078)(13,7.424963471169663)(14,7.42044186169017)(15,7.415293591062979)(16,7.412285161215014)(17,7.414542253030692)(18,7.4062437100074145)
(19,7.404799437816306)(20,7.406033584442134)
    };
    \addlegendentry{FedIT \cite{zhang2023fedgpt}}

    \addplot[
      color=green!55!black, mark=triangle*, thick, mark size=1.5, line width=1.2pt
    ] coordinates {
      (1,8.127783682955497)(2,7.870156332857316)(3,7.7283741452541594)(4,7.643863139763933)(5,7.5861594668451415)(6,7.543705958735825)(7,7.514240298602056)(8,7.491091779993952)(9,7.472315751240525)(10,7.45342862574891)(11,7.438330248779803)(12,7.426316064774477)(13,7.421558292697527)(14,7.415107959291569)(15,7.409293850508142)(16,7.406505050961289)(17,7.407455138338231)(18,7.40393265664712)
(19,7.4041215392124835)(20,7.403620216626298)
    };
    \addlegendentry{Flex-LoRA \cite{bai2024flexlora}}

    \addplot[
      color=red, mark=square*, thick, mark size=1.5, line width=1.2pt
    ] coordinates {
      (1,7.905470203800139)(2,7.747516505882249)(3,7.6617822838015215)(4,7.593893710913728)(5,7.535176618683347)(6,7.505008821971624)(7,7.470340285770151)(8,7.450777759753898)(9,7.4270685958827825)(10,7.414192243895703)(11,7.388155050330754)(12,7.371563893471238)(13,7.361751178520681)(14,7.360059381107212)(15,7.341967113249446)(16,7.329620124850858)(17,7.31593936032524)(18,7.306564763681226)
(19,7.311858939934302)(20,7.28733642454006)
    };
    \addlegendentry{FedFit}

    \end{axis}
  \end{tikzpicture}
  \caption{Evaluation perplexity of different FL methods for Qwen2.5-0.5B-Instruct. The proposed method consistently outperforms all baselines and continues to improve after the others saturate.}
  \label{fig:qwen2.5_0.5b_it_fl_methods_eval_loss}
\end{figure}
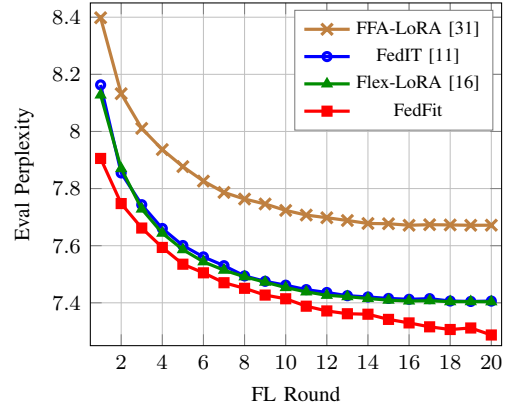

\textbf{Impact of number of clients:} We report the evaluation perplexity of the proposed method and FedIT without quantization on Qwen2.5-0.5B-Instruct for the number of clients varying in \( \{2, 4, 8\} \). One can observe that, for all model sizes, the evaluation perplexity of the proposed method decreases approximately linearly as the number of clients increases, and consistently outperforms the FedIT baseline, whose perplexity reduction is much less pronounced regardless of model size. This observation shows that proposed method is better in terms of scalability, i.e., capable of aggregating information when number of clients is large compared to FedIT.

\begin{table}[t!]
\centering
\small

\renewcommand{\arraystretch}{1.0}
\resizebox{\columnwidth}{!}{%
\begin{tabular}{l l c c c }
\toprule
Model & Method & 2 clients & 4 clients & 8 clients  \\
\midrule
\multirow{3}{*}{\text{Qwen2.5-0.5B-IT}}
  & FedIT            & 7.42  & 7.41   & 7.40 \\
\cmidrule(l){2-5}
  & FedFit       & 7.35  & 7.30  & 7.25 \\
\midrule
\multirow{3}{*}{\text{Qwen2.5-1.5B-IT}}
  & FedIT            & 5.67  & 5.67  & 5.66 \\
\cmidrule(l){2-5}
  & FedFit       & 5.61 & 5.58 & 5.54 \\
\midrule
\multirow{3}{*}{\text{Qwen2.5-3B-IT}}
  & FedIT            & 5.10  & 5.10 & 5.09 \\
\cmidrule(l){2-5}
  & FedFit       & 5.04 & 5.04 & 5.03 \\
\bottomrule
\end{tabular}
}
\caption{Evaluation perplexity of the proposed method compared with the FedIT baseline as a function of the number of clients for different model sizes.}
\label{tab:communication_cost}
\end{table}

\textbf{Impact of vector bank quantization:} We fix the quantization group size to \(64\) and vary the numbers of quantization bits for uplink and downlink, denoted by \(b_{\text{up}}\) and \(b_{\text{down}}\), respectively. The evaluation perplexity under different combinations of quantization bits is shown in Fig.~\ref{fig:qwen2.5_0.5b_it_fl_quant_eval_loss}. The quantization effect is minor when both uplink and downlink use at least \(4\) bits. Even when the uplink is quantized to \(2\) bits, the degradation remains acceptable as long as the downlink uses at least \(4\) bits. However, the evaluation perplexity increases significantly in all configurations with \(2\)-bit downlink quantization. This indicates that the impact of quantization is asymmetric between uplink and downlink, which supports our design choice of decoupling the two directions. Intuitively, downlink broadcast communication is more sensitive because the broadcast vector bank serves as a new anchor for all clients: if the received vector bank deviates significantly from the aggregated one, clients may struggle to continue effective local training, potentially leading to divergence.

In Fig.~\ref{fig:qwen2.5_0.5b_it_fl_u2_eval_loss}, we further illustrate the evaluation perplexity of the proposed method under an extreme \(2\)-bit quantization for uplink transmission. One can observe that the training curves with \(4\)-bit and \(8\)-bit downlink quantization follow similar trajectories from start to end. In contrast, for the configuration with \(2\)-bit quantization on both uplink and downlink, the convergence speed is substantially reduced despite the potential compression ratio of \(8\). Based on these observations, we conclude that a combination of \(4\)-bit RTN quantization for both uplink and downlink is a promising setting for the proposed method, yielding negligible performance degradation while achieving a compression ratio of approximately \(4\times\). This finding further justifies our use of \(4\)-bit RTN in Tab.~\ref{tab:benchmark_score_per_dataset}.
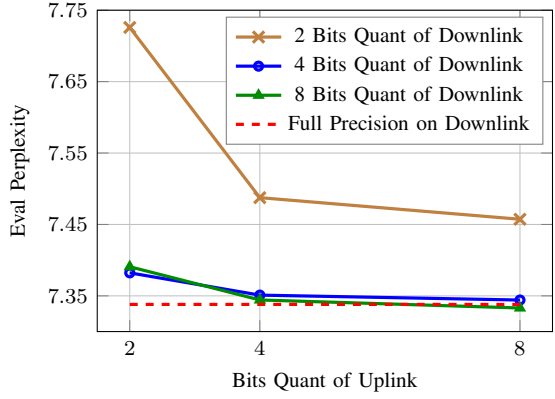
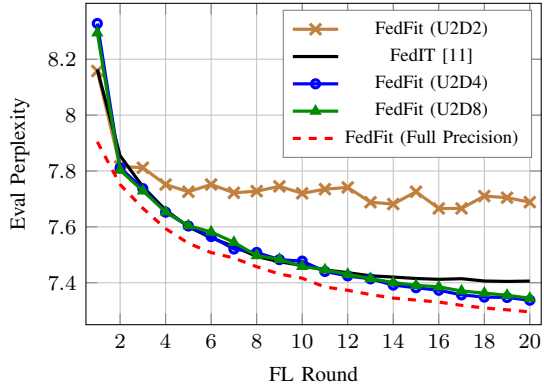
\begin{figure}[t!]
    \centering
    \begin{subfigure}[t]{0.8\linewidth}
        \centering
        \begin{tikzpicture}
    \begin{axis}[
      width=0.85\textwidth,    
      height=0.6\textwidth,
      scale only axis,
      xmin=1.5, xmax=8.5,
      ymin=7.3, ymax=7.75,
      xlabel={Bits Quant of Uplink},
      ylabel={Eval Perplexity},
      xtick={2,4,8},
      ytick={7.35, 7.45,7.55,7.65,7.75},
      grid=both,
      grid style={line width=.1pt, draw=gray!30},
      major grid style={line width=.2pt,draw=gray!50},
      legend style={at={(0.98,0.98)}, anchor=north east, font=\footnotesize, draw=gray},
      tick label style={font=\footnotesize},
      label style={font=\footnotesize},
      clip marker paths=true,
    ]

    \addplot[
      color=brown, mark=x, thick, mark size=3, line width=1.2pt
    ] coordinates {
   (2,7.725842843999727)(4, 7.487533096739514)(8, 7.457278679453329)
    };
    \addlegendentry{2 Bits Quant of Downlink}

    \addplot[
      color=blue, mark=o, thick, mark size=1.5, line width=1.2pt
    ] coordinates {
      (2,7.382241509705612)(4, 7.351037475650238)(8, 7.3441572670226964)
    };
    \addlegendentry{4 Bits Quant of Downlink}

    \addplot[
      color=green!55!black, mark=triangle*, thick, mark size=1.5, line width=1.2pt
    ] coordinates {
      (2,7.3907827551304734)(4, 7.344257949266392)(8, 7.333035569786315)
    };
    \addlegendentry{8 Bits Quant of Downlink}

    \addplot[dashed,
      color=red, thick, mark size=1.5, line width=1.2pt
    ] coordinates {
      (2,7.338027007099737)(4, 7.338027007099737)(8, 7.338027007099737)
    };
    \addlegendentry{Full Precision on Downlink}

    \end{axis}
  \end{tikzpicture}
        \caption{Evaluation perplexity vs.\ different combinations of quantization bits 
        for uplink and downlink transmission for Qwen2.5-0.5B-Instruct.}
        \label{fig:qwen2.5_0.5b_it_fl_quant_eval_loss}
    \end{subfigure}
    
    \vspace{0.5em}
    
    \begin{subfigure}[t]{0.8\linewidth}
        \centering
        \begin{tikzpicture}
    \begin{axis}[
      width=0.85\textwidth,    
      height=0.6\textwidth,
      scale only axis,
      xmin=0.5, xmax=20.5,
      ymin=7.25, ymax=8.4,
      xlabel={FL Round},
      ylabel={Eval Perplexity},
      xtick={2,4,...,20},
      ytick={7.4,7.6,7.8,8.0,8.2},
      grid=both,
      grid style={line width=.1pt, draw=gray!30},
      major grid style={line width=.2pt,draw=gray!50},
      legend style={at={(0.98,0.98)}, anchor=north east, font=\scriptsize, draw=gray},
      tick label style={font=\footnotesize},
      label style={font=\footnotesize},
      clip marker paths=true,
    ]

    \addplot[
      color=brown, mark=x, thick, mark size=3, line width=1.2pt
    ] coordinates {
   (1,8.157838422621076)(2,7.816467223732661)(3,7.8116662247314235)(4,7.751511077674579)(5,7.7257396935578315)(6,7.751666319986695)(7,7.7220364000997)(8,7.727808491445167)(9,7.74495123671032)
(10,7.720094308180356)(11,7.7349291664546636)(12,7.741268264836665)(13,7.688271441138276)(14,7.6816753759574175)(15,7.725842843999727)(16,7.665931868332004)(17,7.6659994935518725)(18,7.710293743356807)(19,7.704393293883051)(20,7.688522569904733)
    };
    \addlegendentry{FedFit (U2D2)}

    \addplot[
      color=black, thick, mark size=1.5, line width=1.2pt
    ] coordinates {
      (1,8.163025388260055)(2,7.85469693564566)(3,7.743544300831668)(4,7.660240694618889)(5,7.6005937981812846)(6,7.560770179577093)(7,7.530043895999796)(8,7.49412147471013)(9,7.475300424831444)(10,7.461708877754723)(11,7.446187170573194)(12,7.436088963248078)(13,7.424963471169663)(14,7.42044186169017)(15,7.415293591062979)(16,7.412285161215014)(17,7.414542253030692)(18,7.4062437100074145)
(19,7.404799437816306)(20,7.406033584442134)
    };
    \addlegendentry{FedIT \cite{zhang2023fedgpt}}
    \addplot[
      color=blue, mark=o, thick, mark size=1.5, line width=1.2pt
    ] coordinates {
      (1,8.32793722306874)(2,7.811219250392772)(3,7.736941396694678)(4,7.6526395731696555)(5,7.6031692608921775)(6,7.565483718034074)(7,7.520776866307417)(8,7.508532846816145)(9,7.482796762162923)(10,7.477500045513124)(11,7.440446258999152)(12,7.424591728135396)(13,7.4140296190107)(14,7.391767834629651)(15,7.382241509705612)(16,7.373340958120437)(17,7.3567278178236055)(18,7.348791831186917)
(19,7.347313217259668)(20,7.337563398422807)
    };
    \addlegendentry{FedFit (U2D4)}
    \addplot[
      color=green!55!black, mark=triangle*, thick, mark size=1.5, line width=1.2pt
    ] coordinates {
      (1,8.294277148646753)(2,7.803045995374514)(3,7.728656066357413)(4,7.653606635938595)(5,7.604285988770862)(6,7.581966294063258)(7,7.544774378824775)(8,7.497234618385312)(9,7.4838297913330445)(10,7.458574920408553)(11,7.446314993944583)(12,7.4305294497045375)(13,7.415180443525486)(14,7.400780617893037)(15,7.3907827551304734)(16,7.385527399356659)(17,7.37062104547918)(18,7.362861412476798)
(19,7.355363347473604)(20,7.344745621058045)
    };
    \addlegendentry{FedFit (U2D8)}

    \addplot[
      color=red, dashed, thick, mark size=1.5, line width=1.2pt
    ] coordinates {
      (1,7.904316783635656)(2,7.751666319986695)(3,7.666666639177811)(4,7.5938466374005715)(5,7.541323243567012)(6,7.507990443767264)(7,7.4885328570002905)(8,7.457515151056161)(9,7.4312327987713)(10,7.416354433195473)(11,7.384868871958333)(12,7.373144071304031)(13,7.357781158571585)(14,7.345158897905117)(15,7.338027007099737)(16,7.330911654143518)(17,7.31927688285435)(18,7.309773399284979)
(19,7.303349700864343)(20,7.295873102163013)
    };
    \addlegendentry{FedFit (Full Precision)}

    \end{axis}
  \end{tikzpicture}
        \caption{Evaluation perplexity vs.\ different combinations of quantization bits for 
        uplink and downlink transmission for Qwen2.5-0.5B-Instruct. ``UxDy'' denotes a combination of $x$ bits for uplink and $y$ bits for downlink.}
        \label{fig:qwen2.5_0.5b_it_fl_u2_eval_loss}
    \end{subfigure}
    
\caption{Evaluation perplexity under different quantization settings for Qwen2.5-0.5B-Instruct. Using 4-bit RTN quantization for both uplink and downlink, which achieves a compression ratio of approximately \(4\times\), results in nearly negligible performance degradation.}
\label{fig:qwen2.5_0.5b_it_fl_quant_combined_tb}
\end{figure}

\begin{table}[t!]
\centering
\small
\setlength{\tabcolsep}{4pt}     
\renewcommand{\arraystretch}{1.0}
\resizebox{\columnwidth}{!}{%
\begin{tabular}{lccc}
\toprule
Scheduling [Both/A(B)] & Communication Cost & Eval Perplexity & Average Score \\
\midrule
\makecell{10}  & 1.05M & 7.25 & 49.73 \\
\makecell{11}  & 0.70M & 7.30 & \textbf{50.69} \\
\makecell{12}  & 0.63M & 7.32 & 50.44 \\
\bottomrule
\end{tabular}
}
\caption{Non-alternating and our alternating scheduling. ``\texttt{xy}'' denotes the number of rounds for the
A+B phase and A/B solo phases respectively.}
\label{tab:scheduler_comparison}
\end{table}

\begin{figure}[t!]
    \centering
    \begin{subfigure}[t]{0.8\linewidth}
        \centering
        \begin{tikzpicture}
    \begin{axis}[
      width=0.85\textwidth,    
      height=0.6\textwidth,
      scale only axis,
      xmin=0.8, xmax=3.2,
      ymin=7.32, ymax=7.425,
      xlabel={Number of Solo Phases},
      ylabel={Eval Perplexity},
      xtick={1,2,3},
      ytick={7.34, 7.36,7.38,7.40,7.42},
      grid=both,
      grid style={line width=.1pt, draw=gray!30},
      major grid style={line width=.2pt,draw=gray!50},
      legend style={at={(0.5,0.75)}, anchor=north east, font=\footnotesize, draw=gray},
      tick label style={font=\footnotesize},
      label style={font=\footnotesize},
      clip marker paths=true,
    ]

    \addplot[
      color=brown, mark=x, thick, mark size=3, line width=1.2pt
    ] coordinates {
   (1,7.410617087124132)(2, 7.415592379578154)(3, 7.4190620364549)
    };
    \addlegendentry{No Both Phase}

    \addplot[
      color=blue, mark=o, thick, mark size=1.5, line width=1.2pt
    ] coordinates {
      (1,7.338027007099737)(2, 7.361825773974477)(3, 7.375041083501319)
    };
    \addlegendentry{One Both Phase}

    \addplot[
      color=green!55!black, mark=triangle*, thick, mark size=1.5, line width=1.2pt
    ] coordinates {
      (1,7.328268544606114)(2, 7.336136892076437)(3, 7.348641153120115)
    };
    \addlegendentry{Two Both Phases}
    \end{axis}
  \end{tikzpicture}
        \caption{Evaluation perplexity vs.\ different scheduling of trainable vector banks
        for Qwen2.5-0.5B-Instruct.}
        \label{fig:qwen2.5_0.5b_it_fl_scheduler_eval_loss}
    \end{subfigure}
    
    \vspace{0.6em}
    
    \begin{subfigure}[t]{0.8\linewidth}
        \centering
        \begin{tikzpicture}
    \begin{axis}[
      width=0.85\textwidth,    
      height=0.6\textwidth,
      scale only axis,
      xmin=0.5, xmax=20.5,
      ymin=7.2, ymax=8.6,
      xlabel={FL Round},
      ylabel={Eval Perplexity},
      xtick={2,4,...,20},
      ytick={7.4,7.6,7.8,8.0,8.2,8.4},
      grid=both,
      grid style={line width=.1pt, draw=gray!30},
      major grid style={line width=.2pt,draw=gray!50},
      legend style={at={(0.98,0.98)}, anchor=north east, font=\scriptsize, draw=gray},
      tick label style={font=\footnotesize},
      label style={font=\footnotesize},
      clip marker paths=true,
    ]
\addplot[
      color=green!55!black, mark=triangle*, thick, mark size=1.5, line width=1.2pt
    ] coordinates {
      (1,8.484506774473425)(2,8.209285205505198)(3,8.081470593905355)(4,7.818565908902073)(5,7.759841260174955)(6,7.712485281270323)(7,7.648395066012796)(8,7.581953640302403)(9,7.545619868957231)(10,7.511290220538206)(11,7.500629811493133)(12,7.480909488239886)(13,7.464612787038552)(14,7.4284782475498865)(15,7.4190620364549)(16,7.396039205453232)(17,7.385083679255458)(18,7.3821825478127785)
(19,7.371350358152691)(20,7.364751389347113)
    };
    \addlegendentry{FedFit (03)}

    \addplot[
      color=blue, mark=o, thick, mark size=1.5, line width=1.2pt
    ] coordinates {
      (1,8.489217309123772)(2,8.209412427488408)(3,7.955972245953012)(4,7.826413941577907)(5,7.71751971871052)(6,7.638068159103063)(7,7.594853349977988)(8,7.57042400298582)(9,7.535322138702796)(10,7.498049753586156)(11,7.469699128836133)(12,7.462605552681468)(13,7.438888902138064)(14,7.4282586359798355)(15,7.415592379578154)(16,7.4071407836713155)(17,7.393773643710149)(18,7.383849502841857)
(19,7.383922561749923)(20,7.372447099065646)
    };
    \addlegendentry{FedFit (02)}
    \addplot[
      color=brown, mark=x, thick, mark size=3, line width=1.2pt
    ] coordinates {
   (1,8.507156619063867)(2,8.021886237165528)(3,7.846386521803325)(4,7.716989816878833)(5,7.651887903129148)(6,7.606193507198934)(7,7.5753639017435095)(8,7.535284411020498)(9,7.517173612956579)(10,7.489689889051394)(11,7.467373615849122)(12,7.460431657877669)(13,7.434408442298661)(14,7.419261918319303)(15,7.410617087124132)(16,7.413059246837423)(17,7.393576211289458)(18,7.378767961751195)
(19,7.375404191057255)(20,7.359319780650084)
    };
    \addlegendentry{FedFit (01)}
\addplot[
      color=purple,dash dot, thick, mark size=1.5, line width=1.2pt
    ] coordinates {
      (1,7.926721688389626)(2,7.730948669391228)(3,7.657956279341364)(4,7.60867109404786)(5,7.553379413827814)(6,7.534549657157998)(7,7.5114352792613515)(8,7.484481083355012)(9,7.467911304109663)(10,7.437362015675464)(11,7.418156444507724)(12,7.394623368402493)(13,7.384975394468159)(14,7.3778224333242495)(15,7.375041083501319)(16,7.35290162750385)(17,7.341515508106482)(18,7.3330373181184445)
(19,7.322512050756677)(20,7.318962780307304)
    };
    \addlegendentry{FedFit (13)}
    \addplot[
      color=black, thick, mark size=1.5, line width=1.2pt
    ] coordinates {
      (1,7.941735849248288)(2,7.76385699490259)(3,7.699964035911345)(4,7.603379541270154)(5,7.576696926348101)(6,7.534030522199003)(7,7.491606170105867)(8,7.464418802000378)(9,7.4407496090297025)(10,7.434651278914894)(11,7.415797472154025)(12,7.397166073958145)(13,7.377743278301508)(14,7.366385429554232)(15,7.361825773974477)(16,7.348093657084292)(17,7.345003040769762)(18,7.333579321172599)
(19,7.321546674324349)(20,7.324146323353665)
    };
    \addlegendentry{FedFit (12)}
    \addplot[
      color=red, dashed, thick, mark size=1.5, line width=1.2pt
    ] coordinates {
      (1,7.904316783635656)(2,7.751666319986695)(3,7.666666639177811)(4,7.5938466374005715)(5,7.541323243567012)(6,7.507990443767264)(7,7.4885328570002905)(8,7.457515151056161)(9,7.4312327987713)
(10,7.416354433195473)(11,7.384868871958333)(12,7.373144071304031)(13,7.357781158571585)(14,7.345158897905117)(15,7.338027007099737)(16,7.330911654143518)(17,7.31927688285435)(18,7.309773399284979)
(19,7.303349700864343)(20,7.295873102163013)
    };
    \addlegendentry{FedFit (11)}

    \end{axis}
  \end{tikzpicture}
        \caption{Evaluation perplexity of different training-part schedulers for
        Qwen2.5-0.5B-Instruct. ``\texttt{xy}'' denotes the number of rounds for the
        A+B phase and A/B solo phases respectively; phase transitions are
        A+B $\rightarrow$ B $\rightarrow$ A.}
        \label{fig:qwen2.5_0.5b_it_fl_scheduler_rd_eval_loss}
    \end{subfigure}
    
    \caption{Comparison of scheduler designs for Qwen2.5-0.5B-Instruct under different
    training-phase schedules. }
    \label{fig:qwen2.5_0.5b_it_fl_scheduler_combined_tb}
\end{figure}
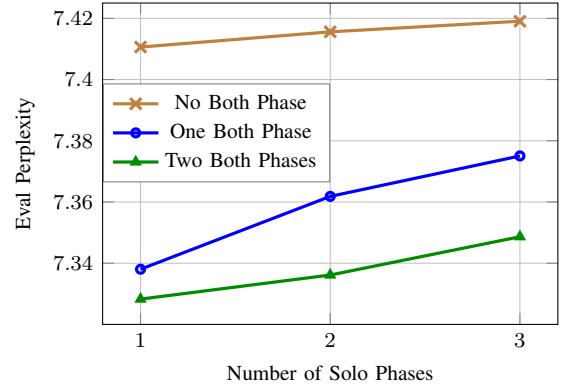
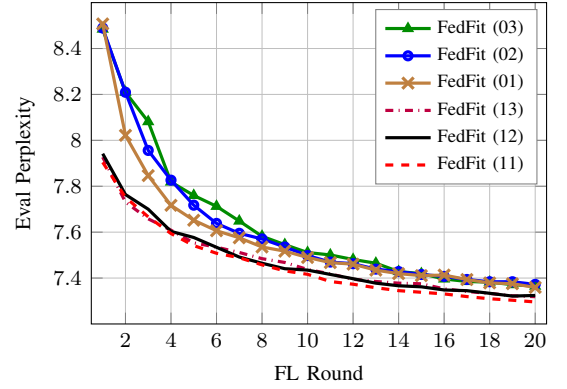

\textbf{Impact of training part scheduler:}
Here we investigate how different scheduling strategies for the trainable vector banks influence training. We first compare the proposed alternating scheduler with a non-alternating variant in which both vector banks remain trainable throughout training, as reported in Tab.~IV. Although the non-alternating approach achieves slightly lower evaluation perplexity than the alternating schedule with one-round both phase and one-round solo phase, its average score on the common-sense reasoning benchmarks is lower, suggesting potential overfitting.

As illustrated in Fig.~\ref{fig:qwen2.5_0.5b_it_fl_scheduler_eval_loss}, the evaluation perplexity decreases as the number of rounds in the both phase increases. The largest improvement comes from introducing a single both phase; additional both-phase rounds yield diminishing returns, which aligns with our design principle of using this phase primarily as a repair/synchronization stage. Meanwhile, the evaluation perplexity increases as the number of solo-phase rounds grows. This behavior is expected, since the vector banks for A and B become progressively more mismatched when one is kept fixed while the other is continually updated.

To further examine the optimal scheduler with at most one both phase, we compare the training trajectories of different schedules for Qwen2.5-0.5B-Instruct in Fig.~\ref{fig:qwen2.5_0.5b_it_fl_scheduler_rd_eval_loss}. One can observe that schedules including one-round both phase consistently outperform those without, confirming the benefit of a repair stage. Moreover, the schedule with one round of both phase and one round of solo phase achieves the best training curve among all tested schedulers with low communication cost. Taken together with the results in Tab.~IV, these findings indicate that the alternating schedule with one round of both phase and one round of solo phase is an effective choice, balancing benchmark performance, communication cost, and training stability. The choice $m_B = m_s = 1$ is structural rather than purely empirical:
$m_B \geq 1$ is required for the bilinear $A$/$B$ coupling to enter the gradient signal at all, as evidenced by the plateau of the ``No Both Phase'' curve in
Fig.~\ref{fig:qwen2.5_0.5b_it_fl_scheduler_eval_loss};
$m_s \geq 1$ ensures SoP-exact aggregation (Sec.~IV.B); and
the minimal configuration satisfying both is the Pareto-optimal point in Tab.~\ref{tab:scheduler_comparison}. We retain a fixed cyclic policy to inherit the convergence behaviour of cyclic block-coordinate descent and to keep the per-round communication budget deterministic; adaptive regime-aware scheduling is left to future work.

\textbf{Per-phase adapter evolution:} We track the two banks over a $T = 12$-round ($4$ meta-cycle) FedFit run on Qwen2.5-1.5B-Instruct. Fig.~\ref{fig:adapter_staircase} plots their norms and round-to-round changes. Three observations confirm the structural argument of Sec.~IV.B: in solo rounds the frozen bank's change is exactly zero, so its aggregation is exact SoP; in joint rounds both banks move, so the SoP-vs-PoS bias arises and the RSC is invoked only there; the solo-round in-phase change is of the same order as the joint-round change, so each solo round is a full-strength training step, not a starved one.

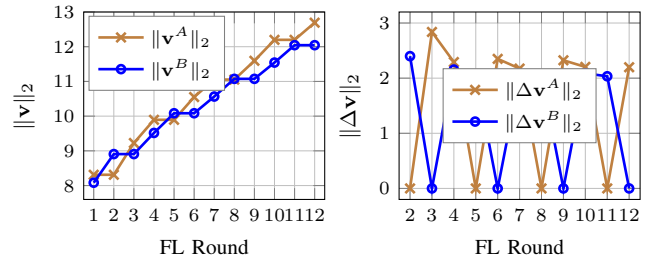
\begin{figure}[t]
\centering
\begin{tikzpicture}
\begin{axis}[
  name=leftax,
  width=0.36\columnwidth, height=0.28\columnwidth,
  scale only axis,
  xmin=0.5, xmax=12.5, ymin=7.6, ymax=13.0,
  xlabel={FL Round}, ylabel={$\|\mathbf v\|_2$},
  xtick={1,2,3,4,5,6,7,8,9,10,11,12},
  ytick={8,9,10,11,12,13},
  grid=both, grid style={line width=.1pt, draw=gray!30},
  major grid style={line width=.2pt, draw=gray!50},
  legend style={at={(0.02,0.98)}, anchor=north west, font=\scriptsize, draw=gray},
  tick label style={font=\scriptsize},
  label style={font=\footnotesize},
  clip marker paths=true,
]
\addplot[color=brown, mark=x, thick, mark size=2.5, line width=1.0pt] coordinates {
  (1,8.311)(2,8.311)(3,9.226)(4,9.894)(5,9.894)(6,10.550)(7,11.053)(8,11.053)(9,11.595)(10,12.198)(11,12.198)(12,12.692)
};
\addlegendentry{$\|\mathbf v^A\|_2$}
\addplot[color=blue, mark=o, thick, mark size=1.5, line width=1.0pt] coordinates {
  (1,8.082)(2,8.907)(3,8.907)(4,9.516)(5,10.083)(6,10.083)(7,10.564)(8,11.074)(9,11.074)(10,11.543)(11,12.044)(12,12.044)
};
\addlegendentry{$\|\mathbf v^B\|_2$}
\end{axis}

\begin{axis}[
  at={(leftax.east)}, anchor=west, xshift=28pt,
  width=0.36\columnwidth, height=0.28\columnwidth,
  scale only axis,
  xmin=1.5, xmax=12.5, ymin=-0.2, ymax=3.2,
  xlabel={FL Round}, ylabel={$\|\Delta\mathbf v\|_2$},
  xtick={2,3,4,5,6,7,8,9,10,11,12},
  ytick={0,1,2,3},
  grid=both, grid style={line width=.1pt, draw=gray!30},
  major grid style={line width=.2pt, draw=gray!50},
  legend style={at={(0.5,0.5)}, anchor=center, font=\scriptsize, draw=gray},
  tick label style={font=\scriptsize},
  label style={font=\footnotesize},
  clip marker paths=true,
]
\addplot[color=brown, mark=x, thick, mark size=2.5, line width=1.0pt] coordinates {
  (2,0.000)(3,2.838)(4,2.292)(5,0.000)(6,2.350)(7,2.174)(8,0.000)(9,2.324)(10,2.204)(11,0.000)(12,2.198)
};
\addlegendentry{$\|\Delta\mathbf v^A\|_2$}
\addplot[color=blue, mark=o, thick, mark size=1.5, line width=1.0pt] coordinates {
  (2,2.399)(3,0.000)(4,2.156)(5,2.070)(6,0.000)(7,2.036)(8,2.050)(9,0.000)(10,2.087)(11,2.034)(12,0.000)
};
\addlegendentry{$\|\Delta\mathbf v^B\|_2$}
\end{axis}
\end{tikzpicture}
\caption{Per-round bank norms (left) and round-to-round changes (right).}
\label{fig:adapter_staircase}
\end{figure}

{\textbf{System-level overhead and wall-clock breakdown:}
Tab.~\ref{tab:wallclock} breaks down the per-round wall-clock over $100$ SFT steps per client on an $8\!\times\!$H200 node: local SFT dominates, and FedFit is slightly faster than FedIT, $23.8$ vs.\ $24.5$~s/round, because its optimizer carries a single vector bank rather than per-layer $A/B$ adapters; its server-side rank-$1$ RSC and blockwise RTN add only a few milliseconds, and the QR$+$SVD aggregator scales sub-quadratically in $K$, under $50$~ms even at $K = 256$. Translating bytes to latency, Tab.~\ref{tab:bw} reports per-round uplink time on a $5$~Mbps cellular link: the LoRA-native baselines scale linearly with model size, reaching $\sim\!16$--$32$~s at $7$B, comparable to a full per-round wall-clock, so communication is the binding constraint there. Because its bank dimension is independent of model size, FedFit's uplink is instead constant at $\sim\!1.12$~s, or $\sim\!0.30$~s with $4$-bit RTN, an order of magnitude below its own local-SFT time, so it mitigates the communication bottleneck at negligible compute cost.

\begin{table}[t]
\centering
\small
\setlength{\tabcolsep}{5pt}
\renewcommand{\arraystretch}{1.0}
{\begin{tabular}{lrr}
\toprule
Phase                              & FedIT       & FedFit      \\
\midrule
Local SFT (per client)             & $24.5$~s    & $23.8$~s    \\
Model build $+$ adapter init       & $<\!0.05$~s & $<\!0.05$~s \\
Central eval (averaged)           & $0.5$~s     & $0.5$~s     \\
Server aggregation                 & $<\!1$~ms   & $2.2$~ms    \\
RTN encode/decode                  & ---         & $<\!10$~ms  \\
Other (RPC, transport)             & $7.4$~s     & $7.4$~s     \\
\midrule
\textbf{Total}                     & $32.4$~s    & $31.7$~s    \\
\bottomrule
\end{tabular}}
\caption{Median per-round wall-clock breakdown for FedIT and FedFit over the run.}
\label{tab:wallclock}
\end{table}

\begin{table}[t]
\centering
\small
\setlength{\tabcolsep}{6pt}
\renewcommand{\arraystretch}{1.0}
{\begin{tabular}{l c c c c}
\toprule
Method            & 0.5B    & 1.5B    & 3B      & 7B \\
\midrule
FedIT~\cite{zhang2023fedgpt}             & 7.04    & 14.77   & 23.95   & 32.35 \\
Flex-LoRA~\cite{bai2024flexlora}         & 7.04    & 23.95   & 23.95   & 32.35 \\
FFA-LoRA~\cite{sun2024ffalora}          & 3.52    & 11.97   & 11.97   & 17.82 \\
RoLoRA~\cite{chen2024rolora}            & 3.55    & 7.41    & 12.02   & 16.18 \\
Quantized-LoRA~\cite{zhu2024quantlora}    & 1.86    & 3.81    & 6.14    & 8.19 \\
LoRA-FAIR~\cite{bian2025lorafair}         & 7.09    & 14.83   & 24.02   & 32.35 \\
LA-LoRA~\cite{liu2026rethinkinglora}           & 7.04    & 14.77   & 23.95   & 32.35 \\
\cmidrule(l){1-5}
\textbf{FedFit}            & \textbf{1.12} & \textbf{1.12} & \textbf{1.12} & \textbf{1.12} \\
\textbf{FedFit (RTN4)}     & \textbf{0.30} & \textbf{0.30} & \textbf{0.30} & \textbf{0.30} \\
\bottomrule
\end{tabular}}
\caption{Uplink time, in seconds, at a representative $5$~Mbps cellular link, derived from the communication cost in Tab.~\ref{tab:communication_cost_And_Bench_score}.}
\label{tab:bw}
\end{table}
}

{\textbf{Impact of vector-bank dimension and LoRA rank:} The vector-bank parameterization introduces two independent design knobs: the total bank dimension $L$ and the LoRA rank $r$ at which the per-layer adapters are generated. We sweep each axis with identical FedFit configuration and report results in Tab.~\ref{tab:rank_bank_sweep}. Doubling $L$ from $524$K to $1.05$M improves evaluation perplexity from $5.62$ to $5.61$ and the $7$-task common-sense average from $60.61$ to $60.68$, while halving it hurts symmetrically, so the default $L = 524$K is not at saturation and the bank dimension is a continuous compression/accuracy lever. By contrast, varying LoRA rank $r \in \{2, 4, 8, 16\}$ at fixed $L = 524$K moves both metrics by at most $0.3$--$0.5$~pt, so at a fixed communication budget FedFit users should prefer larger banks over larger ranks.

\begin{table}[t]
\centering
\small
\setlength{\tabcolsep}{4pt}
\renewcommand{\arraystretch}{1.0}
\resizebox{\columnwidth}{!}{%
{\begin{tabular}{lrrrrr}
\toprule
Configuration                            & $L$         & $r$  & Uplink/round   & Eval ppl        & Avg Score        \\
\midrule
FedFit (smaller bank)                    & 262\,K      & 4    & 0.35\,MB       & $5.67$          & $60.50$          \\
FedFit (default)                         & 524\,K      & 4    & 0.70\,MB       & $\textbf{5.62}$ & $60.61$          \\
FedFit (larger bank)                     & 1.05\,M     & 4    & 1.40\,MB       & $\textbf{5.61}$ & $60.68$          \\
\midrule
FedFit (rank $2$)                        & 524\,K      & 2    & 0.70\,MB       & $5.63$          & $\textbf{60.84}$ \\
FedFit (rank $8$)                        & 524\,K      & 8    & 0.70\,MB       & $5.65$          & $60.73$          \\
FedFit (rank $16$)                       & 524\,K      & 16   & 0.70\,MB       & $5.67$          & $60.53$          \\
\bottomrule
\end{tabular}}%
}
\caption{Sweep over vector-bank dimension and LoRA rank.}
\label{tab:rank_bank_sweep}
\end{table}}

{\textbf{Impact of residual spectral correction:} To isolate RSC's contribution, Tab.~\ref{tab:rsc_ablation} compares five aggregators: no RSC, the correction on $\mathbf{v}^A$, $\mathbf{v}^B$, or both banks at gain $\gamma=0.5$, and a ``pure RSC'' variant that drops the FedAvg anchor at $\gamma=1.0$. The both-banks variant attains the best $7$-task average, while the asymmetric $\mathbf{v}^A$/$\mathbf{v}^B$-only variants trail even the no-RSC baseline: the rank-$1$ SVD of $C_\delta = R_A R_B^\top$ represents the SoP-vs-PoS gap as a single bilinear outer product, which one bank alone cannot recover. Dropping the FedAvg anchor likewise regresses perplexity, confirming that a lone rank-$1$ term cannot represent the full first-order bank update. \new{RSC is defined in the bank space and does not carry over to raw LoRA factors; the lower panel reports one-sided correction there, where neither direction improves on plain averaging.}

\begin{table}[t]
\centering
\small
\setlength{\tabcolsep}{4pt}
\renewcommand{\arraystretch}{1.0}
\resizebox{\columnwidth}{!}{%
{\begin{tabular}{l c c}
\toprule
Aggregator (joint round) & Eval ppl & Average Score \\
\midrule
\multicolumn{3}{l}{\new{\emph{Vector bank (FedFit)}}} \\
no RSC (alt-phases $+$ FedAvg ) & $5.62$ & $60.76$ \\
pure RSC (no FedAvg term, $\gamma = 1.0$)     & $5.72$ & $60.65$ \\
RSC applied to $\mathbf{v}^A$ only            & $5.63$ & $60.57$ \\
RSC applied to $\mathbf{v}^B$ only            & $5.64$ & $60.63$ \\
RSC applied to both banks (proposed)         & $5.64$ & $\mathbf{60.87}$ \\
\midrule
\multicolumn{3}{l}{\new{\emph{Raw LoRA factors}}} \\
\new{no correction (FedAvg of $A$, $B$)}       & \new{$5.68$} & \new{$60.69$} \\
\new{compensate $\bar A$ only}                 & \new{$5.71$} & \new{$60.66$} \\
\new{compensate $\bar B$ only}                 & \new{$5.71$} & \new{$60.54$} \\
\bottomrule
\end{tabular}}%
}
\caption{Ablation of the residual spectral correction\new{, and of one-sided correction on the raw LoRA factors}.}
\label{tab:rsc_ablation}
\end{table}

}

\textbf{Generalization beyond common-sense reasoning:} To test a regime qualitatively different from common-sense reasoning, we federally fine-tune on MetaMathQA~\cite{yu2024metamath} and evaluate on GSM8K~\cite{cobbe2021gsm8k} and MMLU~\cite{hendrycks2021mmlu}, reported in Tab.~\ref{tab:math_reasoning_manu}. FedFit tracks FedIT on both benchmarks at both scales while transmitting $13.2$/$21.4\times$ fewer bytes per round, and $4$-bit RTN adds a further $\sim$$4\times$. The compression mechanism is task-agnostic because only the bank travels, so FedFit extends its communication-efficiency advantage to math-specific fine-tuning at essentially no accuracy cost.

\begin{table}[t!]
\centering
\small
\setlength{\tabcolsep}{5pt}
\renewcommand{\arraystretch}{1.0}
{\begin{tabular}{l l c c}
\toprule
Model & Method & GSM8K & MMLU \\
\midrule
\multirow{3}{*}{\text{Qwen2.5-1.5B-IT}}
 & FedIT         & $64.67$ & $60.12$ \\
\cmidrule(l){2-4}
 & FedFit        & $64.67$ & $59.61$ \\
 & FedFit (RTN4) & $63.38$ & $59.63$ \\
\midrule
\multirow{3}{*}{\text{Qwen2.5-3B-IT}}
 & FedIT         & $74.98$ & $65.46$ \\
\cmidrule(l){2-4}
 & FedFit        & $74.75$ & $65.30$ \\
 & FedFit (RTN4) & $73.62$ & $65.28$ \\
\bottomrule
\end{tabular}}
\caption{MetaMathQA fine-tuning on Qwen2.5-1.5B/3B-Instruct; GSM8K and MMLU scores.}
\label{tab:math_reasoning_manu}
\end{table}

\textbf{Non-IID heterogeneity:} To complement the IID partitioning of Sec.~VII, we sweep a Dirichlet partitioner over Dolly's \texttt{category} field at $\alpha \in \{0.05, 0.1, 0.3, 1.0, 3.0\}$; smaller $\alpha$ corresponds to more skewed client compositions. The $7$-task common-sense average lands in a $0.33$-pt band ($60.60$--$60.93$) across the full sweep, against an IID baseline of $60.61$ (Tab.~\ref{tab:noniid_dirichlet_manu}). The flatness follows structurally from FedFit's shared seeded projections, which absorb client-specific drift before aggregation rather than letting it appear as an aggregation artefact; FedFit's non-IID robustness is therefore theoretically grounded and empirically confirmed.

\begin{table}[t!]
\centering
\small
\setlength{\tabcolsep}{5pt}
\renewcommand{\arraystretch}{1.0}
{\begin{tabular}{l c c c}
\toprule
Dirichlet $\alpha$ & Eval ppl & Average Score & Uplink (MB/rd) \\
\midrule
$0.05$                & $5.63$ & $\mathbf{60.93}$ & $0.70$ \\
$0.1$                 & $5.63$ & $60.60$          & $0.70$ \\
$0.3$                 & $5.65$ & $60.80$          & $0.70$ \\
$1.0$                 & $5.64$ & $\mathbf{60.93}$ & $0.70$ \\
$3.0$                 & $5.63$ & $60.64$          & $0.70$ \\
\midrule
IID baseline         & $5.62$ & $60.61$          & $0.70$ \\
\bottomrule
\end{tabular}}
\caption{Dirichlet non-IID sweep over Dolly's \texttt{category} field.}
\label{tab:noniid_dirichlet_manu}
\end{table}

\textbf{Scalability across client counts:} To address the limited $K = 4$ / $T = 20$ baseline of Tab.~\ref{tab:communication_cost_And_Bench_score}, we sweep $K \in \{4, 16, 32\}$ \new{at a common $T = 50$} and report results in Tab.~\ref{tab:scalability_manu}. The communication cost stays at $\sim$$0.7$~MB across all three cells because the bank dimension is independent of $K$, and both eval perplexity and the $7$-task common-sense average remain flat across the $8\times$ scaling\new{, so FedFit remains stable up to the tested full-participation client counts. Partial participation, client sampling and straggler effects are not evaluated here and are left to future work}.

\begin{table}[t!]
\centering
\small
\setlength{\tabcolsep}{5pt}
\renewcommand{\arraystretch}{1.0}
{\begin{tabular}{l c c c}
\toprule
Clients $K$ & Rounds $T$ & Eval ppl & Average Score \\
\midrule
$ 4$ & \new{$50$} & \new{$5.63$} & \new{$60.54$}           \\
$16$            & $50$ & $5.60$ & $\mathbf{60.96}$  \\
$32$            & $50$ & $5.61$ & $60.47$           \\
\bottomrule
\end{tabular}}
\caption{Scalability of FedFit across client counts $K \in \{4, 16, 32\}$.}
\label{tab:scalability_manu}
\end{table}

\new{Tab.~\ref{tab:rsc_spectrum} reports the spectrum of the core matrix $\mathcal{C}_\delta$ behind these cells. The top mode carries $36.3\%$ of the residual energy at $K = 4$ but only $7.9\%$ and $4.1\%$ at $K = 16$ and $K = 32$, so a rank-$1$ correction is faithful only at $K = 4$; Algorithm~\ref{alg:residual_spectral_Agg} returns FedAvg at the larger counts.}

\begin{table}[t!]
\centering
\small
\setlength{\tabcolsep}{5pt}
\renewcommand{\arraystretch}{1.0}
\new{\begin{tabular}{c c c c c}
\toprule
Clients $K$ & Joint rounds & mean $\sigma_1$ & mean $\rho_1$ & Correction \\
\midrule
$ 4$ & $17$ & $4.09$ & $0.364$ & applied  \\
$16$ & $17$ & $1.07$ & $0.079$ & withheld \\
$32$ & $17$ & $0.56$ & $0.041$ & withheld \\
\bottomrule
\end{tabular}}
\caption{\new{Spectrum of the joint-round core matrix and the resulting gate decision, averaged over the joint rounds of each run.}}
\label{tab:rsc_spectrum}
\end{table}

\section{Conclusions}

This work addresses the critical communication and aggregation bottlenecks in Federated LLM fine-tuning, offering three key takeaways for future edge intelligence systems. First, we demonstrate that the high dimensionality of LoRA adapters is largely redundant for LLMs; by using a {disjoint shared vector-bank parameterization}, we show that distinct client models can be effectively constructed from a shared, ultra-compact parameter space without sacrificing expressivity. Second, we resolve the fundamental ``aggregation dilemma'' (Sum-of-Products vs. Product-of-Sums) not only through approximation, but also through a novel {alternating optimization schedule}. This finding suggests that sequential updating of low-rank factors is a viable strategy to achieve accurate aggregation in distributed settings. Finally, the success of {FedFit} proves that extreme compression (reducing communication costs to just $1$--$3\%$ of standard methods) is achievable when structural parameter sharing is coupled with post-training quantization. These results constitute important first steps to contribute to deploy billion-parameter models over highly constrained wireless networks. Other issues not addressed in this paper would need to be investigated separately. These include a deeper analysis of the LLM model's effect on the results, the study of scenarios with client heterogeneity and partial participation, accounting for transmission errors and delays, and the acquisition of new data.

\appendices
\section{Proof to Theorem \ref{lem:1}}

We first show the proof regarding $\mathbf{v}^B$ with fixed $\mathbf{v}^A$. It is trivial to prove the remaining part following the same procedure as below. In the following, we denote $\bar{d}_A:=\sum_{\ell=1}^L rd_\ell^\mathrm{in}$, $\bar{d}_B:=\sum_{\ell=1}^L rd_\ell^\mathrm{out}$, and $\bar{d}_W:=\sum_{\ell=1}^L d_\ell^\mathrm{in}d_\ell^\mathrm{out}$ as the amount of adjustable parameters corresponding to all $A_\ell$s, $B_\ell$s, and $W_\ell$s, respectively. Below, we only show the proof sketch due to the space limits.

\textit{Step 1: Writing the mapping $\mathbf{v}^B\mapsto \Delta w$ as a linear operator:}
We first denote
\begin{equation}
D_A := \mathrm{diag}\big(A_1\otimes I_{d^\textrm{out}_1},\dots,A_L\otimes I_{d^\textrm{out}_L}\big)\in \mathbb R^{\bar{d}_W\times \bar{d}_B},
\end{equation}
where $\otimes$ stands for the Kronecker product.
Recalling the updating procedure in \eqref{eq:lora_update}, we can verify that
\begin{equation}
\Delta \mathcal{W} = \overline{P}^B(\mathbf{v}^A)\mathbf{v}^B,\qquad
\overline{P}^B(\mathbf{v}^A) := \frac{\alpha}{r} D_A P^B \in \mathbb R^{\bar{d}_W\times d_B},
\end{equation}
which is a linear map from the low-dimensional parameter space $\mathbb R^{d_B}$ to the full stacked update space $\mathbb R^{d_W}$. In the following, for simplification, we may omit the variable $\mathbf{v}^A$ in $\overline{P}^B(\mathbf{v}^A)$.

\noindent\textit{Step 2: Estimating the operator norm of $\overline{P}^B$:}
We first note from the definition that in $P^B\in\mathbb R^{\bar{d}_B\times d_B}$, each row is a randomly generated one-hot vector. Thus, its operator norm can be represented by
\begin{equation}
\|P^B\|_\mathrm{op}=\max_{x\in \mathbb{R}^{d_B}}\frac{\|P^B\|_2}{\|x\|_2}=\max_{j\in d_B} \left(\sum_{i=1}^{\bar{d}_B}[P^B]_{i,j}\right)^{1/2}.
\end{equation}
From the definition of $P^B$, we know
\begin{equation}
\sum_{i=1}^{\bar{d}_B}[P^B]_{i,j} = 1, \textrm{ thus, }\|P^B\|_\mathrm{op}= 1.
\end{equation}

Since we always have $\bar{d}_B \gg d_B$ in LLMs, and $\bar{d}_B$ is always sufficiently large, from the scaling law, we can assert
\begin{equation}
\|A_\ell\|^2_{\mathrm{op}}\le\|A_\ell\|_2^2\approx r d^\textrm{in}_\ell / \bar{d}_A.
\end{equation}
Thus,
\begin{equation}
\|D_A\|_{\mathrm{op}}=\max_\ell\|A_\ell\|_{\mathrm{op}}\le \sqrt{r (\max_\ell d^\textrm{in}_\ell) / \bar{d}_A}.
\end{equation}

Combining the above estimates, we know
\begin{equation}\label{eq:L_norm_Bound}
\begin{aligned}
&\|\overline{P}^B\|^2_{\mathrm{op}}
= \Big\|\frac{\alpha}{r} D_A P^B\Big\|^2_{\mathrm{op}}\\
&\le \frac{\alpha^2}{r^2}\,\|D_A\|^2_{\mathrm{op}}\,\|P^B\|^2_{\mathrm{op}}
\le \frac{\alpha^2(\max_\ell d^\textrm{in}_\ell)}{r\bar{d}_A}\|\mathbf{v}^A\|^2_2.
\end{aligned}
\end{equation}

\noindent\textit{Step 3: Smoothness of $g=f\circ \overline{P}^B$ on $\mathbf{v}^B$:}
Define $g(\mathbf{v}^B)=f(\overline{P}^B\mathbf{v}^B)$ for $\mathbf{v}^B\in\mathbb R^{k_B}$. 
For any $\mathbf{v}^B_1,\mathbf{v}^B_2\in\mathbb R^{d_B}$,
\begin{equation}
\begin{aligned}
&\|\nabla g(\mathbf{v}^B_1) - \nabla g(\mathbf{v}^B_2)\|_2\\
&= \big\|(\overline{P}^B)^\top\big(\nabla f(\overline{P}^B\mathbf{v}^B_1) - \nabla f(\overline{P}^B\mathbf{v}^B_2)\big)\big\|_2 \\
&\le \|(\overline{P}^B)^\top\|_{\mathrm{op}}\,\beta\,\|\overline{P}^B\mathbf{v}^B_1 - \overline{P}^B\mathbf{v}^B_2\|_2 \\
&\le \beta\|\overline{P}^B\|_\mathrm{op}^2\,\|\mathbf{v}^B_1 - \mathbf{v}^B_2\|_2.\label{eq:beta_Bound}
\end{aligned}
\end{equation}
Thus $g$ is $\beta_B$-smooth with
$
\beta_B = \beta\|\overline{P}^B\|_{\mathrm{op}}^2$.

Bringing \eqref{eq:L_norm_Bound} into \eqref{eq:beta_Bound}, we can easily obtain the smoothness result in Theorem \ref{lem:1} regarding matrix $B$. And that regarding matrix $A$ can be similarly obtained by the symmetry.

{\noindent\textit{Step 4: Mixed Smoothness: }
By the chain rule,
\begin{equation}
\nabla_{\mathbf v^B} f(\mathbf v^A,\mathbf v^B)
=
\left[\overline{P}^B(\mathbf v^A)\right]^\top
\nabla f(\Delta\mathcal W(\mathbf v^A,\mathbf v^B)).
\end{equation}
Therefore,
\begin{equation}
\begin{aligned}
&\nabla_{\mathbf v^B} f(\mathbf v^{A\prime},\mathbf v^B)
-
\nabla_{\mathbf v^B} f(\mathbf v^A,\mathbf v^B)\\
=&
\left[\overline{P}^B(\mathbf v^{A\prime})-\overline{P}^B(\mathbf v^A)\right]^\top
\nabla f(\Delta\mathcal W(\mathbf v^{A\prime},\mathbf v^B))\\
&+
\overline{P}^B(\mathbf v^A)^\top
\left[
\nabla f(\Delta\mathcal W(\mathbf v^{A\prime},\mathbf v^B))
-
\nabla f(\Delta\mathcal W(\mathbf v^A,\mathbf v^B))
\right].
\end{aligned}
\end{equation}
Since the LoRA mapping is bilinear, $\overline{P}^B(\mathbf{v}^A)$ is linear in $\mathbf v^A$. Together with $\|P^A\|_{\rm op}=\|P^B\|_{\rm op}=1$, this gives
\begin{equation}
\|\overline{P}^B(\mathbf{v}^{A\prime})-\overline{P}^B(\mathbf{v}^A))\|_{\rm op}
\lesssim
\frac{\alpha}{r}
\|\mathbf v^{A\prime}-\mathbf v^A\|_2.
\end{equation}
Using the bounded-gradient assumption and the $\beta$-smoothness of the full-space objective, we obtain
\begin{equation}
\begin{aligned}
&\|\nabla_{\mathbf v^B} f(\mathbf v^{A\prime},\mathbf v^B)
-
\nabla_{\mathbf v^B} f(\mathbf v^A,\mathbf v^B)\|_2\\
\le\;&
\frac{\alpha G}{r}
\|\mathbf v^{A\prime}-\mathbf v^A\|_2\\
&+
\beta
\|\overline{P}^B(\mathbf v^A)\|_{\rm op}
\|\overline{P}^A(\mathbf v^B)\|_{\rm op}
\|\mathbf v^{A\prime}-\mathbf v^A\|_2\\
\lesssim\;&
\left(
\frac{\alpha G}{r}
+
\sqrt{\beta_A\beta_B}
\right)
\|\mathbf v^{A\prime}-\mathbf v^A\|_2,
\end{aligned}
\end{equation}
where $\overline{P}^A(\mathbf v^B)$ has a symmetric definition and norm bound with $\overline{P}^B(\mathbf v^A)$.
Thus,
\begin{equation}
\|\nabla_{\mathbf v^B} f(\mathbf v^{A\prime},\mathbf v^B)
-
\nabla_{\mathbf v^B} f(\mathbf v^A,\mathbf v^B)\|_2
\le
\beta_{\rm mix}
\|\mathbf v^{A\prime}-\mathbf v^A\|_2,
\end{equation}
where
\([)
\beta_{\rm mix}
\lesssim
\frac{\alpha G}{r}
+
\sqrt{\beta_A\beta_B}.
\)
The proof of the other mixed-smoothness direction is symmetric. This completes the proof.
}

\section{Proof to Theorem \ref{the:2}}\label{proof:2}
{
Denote
}
\begin{equation}\begin{aligned}
&\mathbf g_t^A=
\nabla_{\mathbf v^A}f(\mathbf v_t^A,\mathbf v_t^B),
\quad
\mathbf g_t^B=
\nabla_{\mathbf v^B}f(\mathbf v_t^A,\mathbf v_t^B),\\
&\widetilde{\mathbf g}_t^B=
\nabla_{\mathbf v^B}f(\mathbf v_{t+1}^A,\mathbf v_t^B).
\end{aligned}
\end{equation}
Using the block-wise smoothness along
$(\mathbf v_t^A,\mathbf v_t^B)\rightarrow
(\mathbf v_{t+1}^A,\mathbf v_t^B)\rightarrow
(\mathbf v_{t+1}^A,\mathbf v_{t+1}^B)$, we have
\begin{equation}\footnotesize
\begin{aligned}
&f(\mathbf v^A_{t+1},\mathbf v^B_{t+1})
-
f(\mathbf v^A_t,\mathbf v^B_t)\\
\le\;&
-\eta_t\|\mathbf g_t^A\|_2^2
+
\frac{\beta_A\eta_t^2}{2}\|\mathbf g_t^A\|_2^2
-
\eta_t\langle \widetilde{\mathbf g}_t^B,\mathbf g_t^B\rangle
+
\frac{\beta_B\eta_t^2}{2}\|\mathbf g_t^B\|_2^2 .
\end{aligned}
\end{equation}
{
We decompose the inner product as
}
\begin{equation}
{
-\eta_t\langle \widetilde{\mathbf g}_t^B,\mathbf g_t^B\rangle
=
-\eta_t\|\mathbf g_t^B\|_2^2
-
\eta_t
\langle \widetilde{\mathbf g}_t^B-\mathbf g_t^B,\mathbf g_t^B\rangle .
}
\end{equation}
{
By the mixed-smoothness in Theorem \ref{lem:1},
}
\begin{equation}
{
\|\widetilde{\mathbf g}_t^B-\mathbf g_t^B\|_2
\le
\beta_{\rm mix}
\|\mathbf v_{t+1}^A-\mathbf v_t^A\|_2
=
\beta_{\rm mix}\eta_t\|\mathbf g_t^A\|_2 .
}
\end{equation}
{
Thus,
}
\begin{equation}
{
\begin{aligned}
-\eta_t
\langle \widetilde{\mathbf g}_t^B-\mathbf g_t^B,\mathbf g_t^B\rangle
&\le
\eta_t
\|\widetilde{\mathbf g}_t^B-\mathbf g_t^B\|_2
\|\mathbf g_t^B\|_2\\
&\le
\beta_{\rm mix}\eta_t^2
\|\mathbf g_t^A\|_2
\|\mathbf g_t^B\|_2\\
&\le
\frac{\beta_{\rm mix}\eta_t^2}{2}\|\mathbf g_t^A\|_2^2
+
\frac{\beta_{\rm mix}\eta_t^2}{2}\|\mathbf g_t^B\|_2^2 .
\end{aligned}
}
\end{equation}
{
Substituting it back gives
}
\begin{equation}
{
\begin{aligned}
&f(\mathbf v^A_{t+1},\mathbf v^B_{t+1})
-
f(\mathbf v^A_t,\mathbf v^B_t)\\
\le\;&
-\eta_t\left(1-\frac{(\beta_A+\beta_{\rm mix})\eta_t}{2}\right)
\|\mathbf g_t^A\|_2^2\\
&-
\eta_t\left(1-\frac{(\beta_B+\beta_{\rm mix})\eta_t}{2}\right)
\|\mathbf g_t^B\|_2^2\\
\le\;&
-c_t^{\rm AB}
\|\nabla f(\mathbf v_t^A,\mathbf v_t^B)\|_2^2 ,
\end{aligned}
}
\end{equation}
{
where $c_t^{\rm AB}$ is defined in the theorem. When $\eta_t$ is sufficiently small, $c_t^{\rm AB}>0$. Summing the above inequality from $t=0$ to $T-1$ yields
}
\begin{equation}\footnotesize
{
\sum_{t=0}^{T-1}
c_t^{\rm AB}
\|\nabla f(\mathbf v_t^A,\mathbf v_t^B)\|_2^2
\le
f(\mathbf v_0^A,\mathbf v_0^B)
-
f(\mathbf v_T^A,\mathbf v_T^B)
\le
f(\mathbf v_0^A,\mathbf v_0^B)-f_\star .
}
\end{equation}
{
Therefore,
}
\begin{equation}
{
\min_{0\le t<T}
\|\nabla f(\mathbf v_t^A,\mathbf v_t^B)\|_2^2
\le
\frac{
f(\mathbf v_0^A,\mathbf v_0^B)-f_\star
}{
\sum_{t=0}^{T-1} c_t^{\rm AB}
}.
}
\end{equation}
{
For a constant learning rate
$0<\eta<2/(\max\{\beta_A,\beta_B\}+\beta_{\rm mix})$,
we have
}
\begin{equation}
{
c_t^{\rm AB}
=
\eta\left(1-\frac{(\max\{\beta_A,\beta_B\}+\beta_{\rm mix})\eta}{2}\right),
}
\end{equation}
{
which gives the rate $\mathcal O(1/T)$. Hence the A+B phase converges to a stationary point in the sense that the gradient norm vanishes along the generated sequence. If the iterates converge to $(\mathbf v^A,\mathbf v^B)$, then by continuity of the gradient,
$\nabla f(\mathbf v^A,\mathbf v^B)=0$.
This completes the proof.
}
\printbibliography

\end{document}